\documentclass[twoside,11pt]{article} 

\usepackage{jair, theapa, rawfonts}
\jairheading{58}{2017}{231-266}{09/16}{01/17}
\ShortHeadings{DESPOT: Online POMDP Planning with Regularization}
{Ye, Somani, Hsu, \& Lee}
\firstpageno{231}

\usepackage{times}

\usepackage{pdfsync}

\usepackage{url}
\usepackage{xspace}

\usepackage{amsmath}
\usepackage{ntheorem}
\usepackage[ruled]{algorithm}
\usepackage[noend]{algorithmic}
\usepackage{multicol}
\usepackage{caption}
\usepackage{amsfonts}
\usepackage{bm}
\usepackage{upgreek}
\usepackage{booktabs}
\usepackage{graphicx}
  \graphicspath{{figs/}}
  \DeclareGraphicsExtensions{.pdf,.jpg,.png,.eps}
 \usepackage{accents}

\newcommand{\algref}[1]{Algorithm~\ref{#1}}
\renewcommand{\eqref}[1]{(\ref{#1})}
\newcommand{\figref}[1]{Figure~\ref{#1}}
\newcommand{\secref}[1]{Section~\ref{#1}}

\newcommand{\subfig}[1]{\textit{#1}}

\newcommand{\ie}{\textrm{i.e.}}
\newcommand{\eg}{\textrm{e.g.}}

\newenvironment{dm}
  {\vspace*{0pt}\displaymath}{\vspace*{0pt}\enddisplaymath}
\newenvironment{eq}
  {\vspace*{0pt}\equation}{\vspace*{0pt}\endequation}
\newcommand{\sss}{\scriptscriptstyle}

\newlength{\citeskipup}
\newlength{\citeskipdown}

\usepackage{color}
\definecolor{darkgreen}{rgb}{0,0.5,0}
\newcounter{cmt}

\DeclareMathOperator*{\argmax}{arg\,max}
\renewcommand{\O}{\ensuremath{\mathcal O}}

\newtheorem{theorem}{Theorem}[section]
\newtheorem*{theoremstar}{Theorem}
\newtheorem*{lemmastar}{Lemma}
\newtheorem{lemma}{Lemma}[section]
\newenvironment{proof}{\noindent{\bf Proof.}}{$\Box$\medskip}
\makeatletter 
\newcommand{\bfgreek}[1]{\bm{\@nameuse{up#1}}}
\makeatother

\newcommand{\bvar}[1]{{#1}}
\newcommand{\children}{\mathrm{CH}\xspace}
\newcommand{\depth}{\ensuremath{\Delta}\xspace}
\newcommand{\despot}{DESPOT\xspace}
\newcommand{\dtree}{\ensuremath{\mathcal{D}}\xspace}

\newcommand{\ev}{\mbox{$\textbf{\textrm{E}}$}\xspace}

\newcommand{\height}{\ensuremath{D}\xspace}
\newcommand{\len}{\ensuremath{\ell}\xspace}

\newcommand{\node}{\ensuremath{b}\xspace}
\newcommand{\nsam}{\ensuremath{K}\xspace}
\newcommand{\pol}{\ensuremath{\pi}\xspace}
\newcommand{\polclass}{\ensuremath{\Pi}\xspace}
\newcommand{\pr}{\ensuremath{p}}
\newcommand{\rmax}{\ensuremath{R_\mathrm{max}}}
\newcommand{\seqset}{\ensuremath{{\boldsymbol{\Phi}}}\xspace}
\newcommand{\seq}{\ensuremath{{\boldsymbol{\phi}}}\xspace}
\newcommand{\seqt}{\ensuremath{\phi}}
\newcommand{\wdeu}{\ensuremath{E}\xspace}
\newcommand{\gap}{\ensuremath{\epsilon}\xspace}

\newcommand{\tmax}{\ensuremath{T_\mathrm{max}}\xspace}

\newcommand{\lowerv}{L} 
\newcommand{\upperv}{U} 
\newcommand{\lowernu}{\ell} 
\newcommand{\uppernu}{\mu} 

\begin{document}

\title{DESPOT: Online POMDP Planning with Regularization}
\author{\name Nan Ye \email n.ye@qut.edu.au \\
\addr ACEMS \& Queensland University of Technology, Australia \\
   \\
\name Adhiraj Somani \email adhirajsomani@gmail.com\\
\name David Hsu \email dyhsu@comp.nus.edu.sg\\
\name Wee Sun Lee \email leews@comp.nus.edu.sg\\
\addr
National University of Singapore,
Singapore
}
\maketitle

\begin{abstract}
  The partially observable Markov decision process (POMDP) provides a
  principled general framework for planning under uncertainty, but solving
  POMDPs optimally is
  computationally intractable,  due to the ``curse of dimensionality''
  and the ``curse of history''.  To overcome these challenges, we
  introduce the \emph{Determinized Sparse Partially Observable Tree}
  (\despot),   a sparse approximation of the standard belief tree,
  for online planning under uncertainty.
  A DESPOT focuses online planning on a set of
  randomly sampled \emph{scenarios}  and compactly
  captures  the ``execution'' of all policies under these scenarios.
  We show that the best policy obtained
  from a DESPOT is near-optimal, with a regret bound that depends on
  the representation size
  of the optimal policy.  Leveraging this result, we give
  an anytime online
  planning algorithm, which searches a DESPOT for a policy that optimizes a
  regularized objective function. Regularization  balances
  the estimated  value of a policy under the sampled scenarios and
 the policy size, thus avoiding  overfitting. 
 The algorithm demonstrates strong experimental results, compared with some of
 the best online POMDP algorithms available. It has also been incorporated
 into an autonomous driving system for real-time vehicle control. The source code for the algorithm is available online.
 \end{abstract}

\section{Introduction}
\label{sec:intro}

The partially observable Markov decision process (POMDP) \cite{SmaSon73} provides a
principled general framework for planning in partially observable
stochastic environments.  It has a wide range of applications ranging from
robot control~\cite{roy1999coastal}, resource management~\cite{ChaCar12} to
medical diagnosis~\cite{hauskrecht2000planning}.
However, solving POMDPs optimally is computationally intractable
\cite{PapTsi87,MadHan99}. Even approximating the optimal solution is
difficult~\cite{LusGol01}.  There has been substantial progress
in the last
decade~\cite{pingor03,SmiSim04,poupart2005exploiting,KurHsu08,SilVen10,BaiHsu11}.
However, it remains a challenge today to scale up POMDP planning and tackle
POMDPs with very large state spaces and complex dynamics to achieve near
real-time performance in practical applications (see, \eg,
Figures~\ref{fig:driving} and \ref{fig:demining}).

Intuitively, POMDP planning faces two main difficulties. One is the ``curse of
dimensionality''.  A large state space  is a well-known difficulty for
planning: the number of states grows exponentially with the
number of state variables. Furthermore, in a partially observable environment,
the agent must reason in the space of \emph{beliefs},
which are probability distributions over the states.
If one naively  discretizes the belief space, the number of
discrete beliefs then grows exponentially with the number of states.
The second difficulty is the ``curse of history'': the number of
action-observation histories under consideration for POMDP planning grows
exponentially with the planning horizon.  Both curses result in exponential
growth of computational complexity and major barriers to large-scale POMDP
planning.

\begin{figure}
  \begin{center}
    \begin{tabular}{c@{\hspace*{0.2in}}c}
      \includegraphics[width=2.9in]{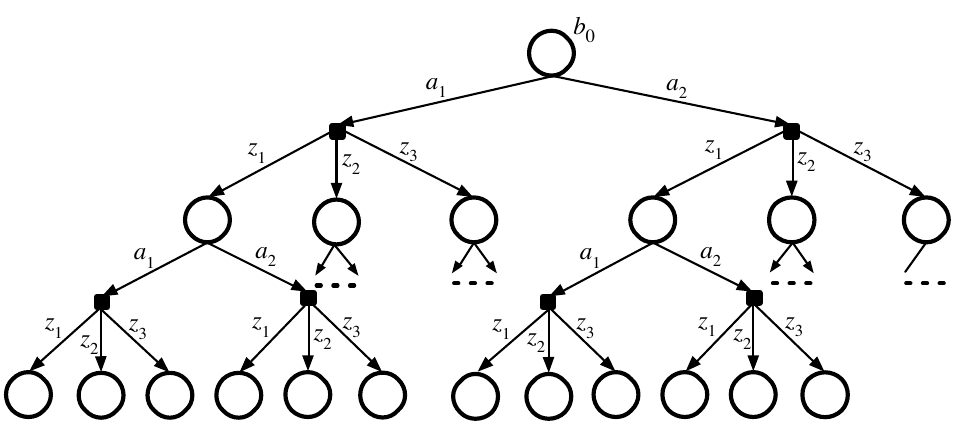} 
      &
      \includegraphics[width=2.9in]{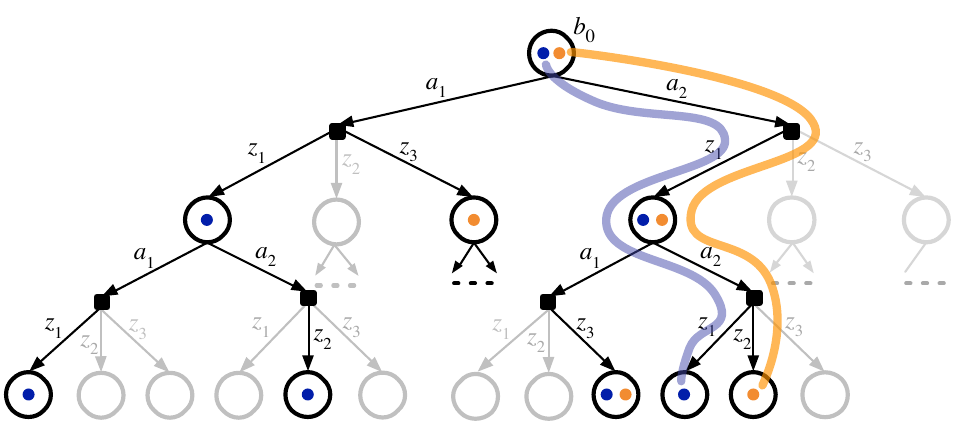} \\
      (\subfig a) & (\subfig b)
    \end{tabular}
\caption{Online POMDP planning performs lookahead search on a tree.
  (\subfig{a}) A standard belief tree of height $D=2$.
  It contains all action-observation
        histories. Each belief tree node represents a belief. Each path
        represents an action-observation history. 
 (\subfig{b})
  A DESPOT (black),   obtained under 2 sampled scenarios marked with blue and
  orange dots,  is overlaid on the standard belief tree. 
       A \despot contains only the histories under the sampled scenarios. 
          Each \despot node represents a belief implicitly and  contains
          a particle set that  approximates the belief. The
          particle set consists of a subset of sampled   scenarios.
     }
	\label{fig:despot}
\end{center} \end{figure}

This paper presents a new anytime online algorithm\footnote{The source code for the algorithm is available at
\url{http://bigbird.comp.nus.edu.sg/pmwiki/farm/appl/}.} for  POMDP planning.  Online
POMDP planning chooses one action at a time and interleaves planning
and plan execution \cite{RosPin08}. At each time step, the
agent performs lookahead search on a belief tree
rooted at the current belief (\figref{fig:despot}\subfig a) and 
executes  immediately the best action found.
To attack the two curses, our algorithm exploits two basic ideas:
sampling and anytime heuristic search, described below.

To speed up lookahead search,  we introduce
the \emph{Determinized Sparse Partially Observable Tree} (\despot) as a sparse
approximation to  the standard belief tree.
The construction of a \despot leverages a set of randomly sampled
\emph{scenarios}. 
Intuitively,  a scenario is a sequence of random numbers that determinizes the
execution of a policy under uncertainty and  generates a unique trajectory
of states and observations given an action sequence. 
A \despot encodes the execution of all policies under a fixed set of
$\nsam$ sampled scenarios (\figref{fig:despot}\subfig b).  It alleviates the
``curse of dimensionality'' by sampling states from a belief and
alleviates the ``curse of history'' by sampling observations.  A
standard belief tree of height $\height$ contains all possible
action-observation histories and thus $\O(|A|^\height|Z|^\height)$
belief nodes, where $|A|$ is the number of actions and $|Z|$ is the number of
observations.  In contrast, a corresponding \despot contains only histories
under the $\nsam$ sampled scenarios and thus $\O(|A|^\height \nsam)$
belief nodes.
A \despot is much sparser than a standard belief tree
when $\nsam$ is small but converges to the standard belief tree as
$\nsam$ grows.

To approximate the standard belief tree well, the number of scenarios, $\nsam$,
required by a \despot may be exponentially large in the worst case.
The worst case, however, may be uncommon in practice.  
We give a competitive analysis that
compares our computed policy against near optimal policies and show that
$\nsam \in \O(|\pol| \ln (|\pol| |A| |Z|) )$ is
sufficient to guarantee near-optimal performance for the lookahead
search, when a POMDP admits a near-optimal policy $\pol$ with
representation size $|\pol|$ (Theorem~\ref{th:despot}).  Consequently,
if a POMDP admits a compact near-optimal policy, it is
sufficient to use a DESPOT  much smaller than a standard belief
tree, significantly speeding up the lookahead search.  Our
experimental results support this view in practice: $K$ as small as
500 can work well for some large POMDPs.

As a \despot is constructed from sampled scenarios, lookahead search
may overfit the sampled scenarios and find a biased policy. To
achieve the desired performance bound, our algorithm uses
\emph{regularization}, which balances the estimated value of a policy
under the sampled scenarios and the size of the policy.
Experimental results show that when  overfitting occurs,
regularization is helpful in practice.

To further improve the practical performance of lookahead search, we
introduce an anytime heuristic algorithm to search a \despot.
The algorithm constructs a \despot incrementally under the guidance of a
heuristic.  Whenever the maximum planning time is reached, it outputs the
action with the best regularized value based on the partially constructed
\despot, thus endowing the algorithm with the \emph{anytime} characteristic.
We show that if the search heuristic is admissible,
our algorithm eventually finds the optimal
action, given sufficient time.  We further show that if the heuristic is not
admissible, the performance of the algorithm degrades gracefully.
Graceful degradation is important, as it allows practitioners to use
their intuition to design good heuristics that are not necessarily
admissible over the entire belief space.

\begin{figure}
\begin{center}
\includegraphics[width=3.5in]{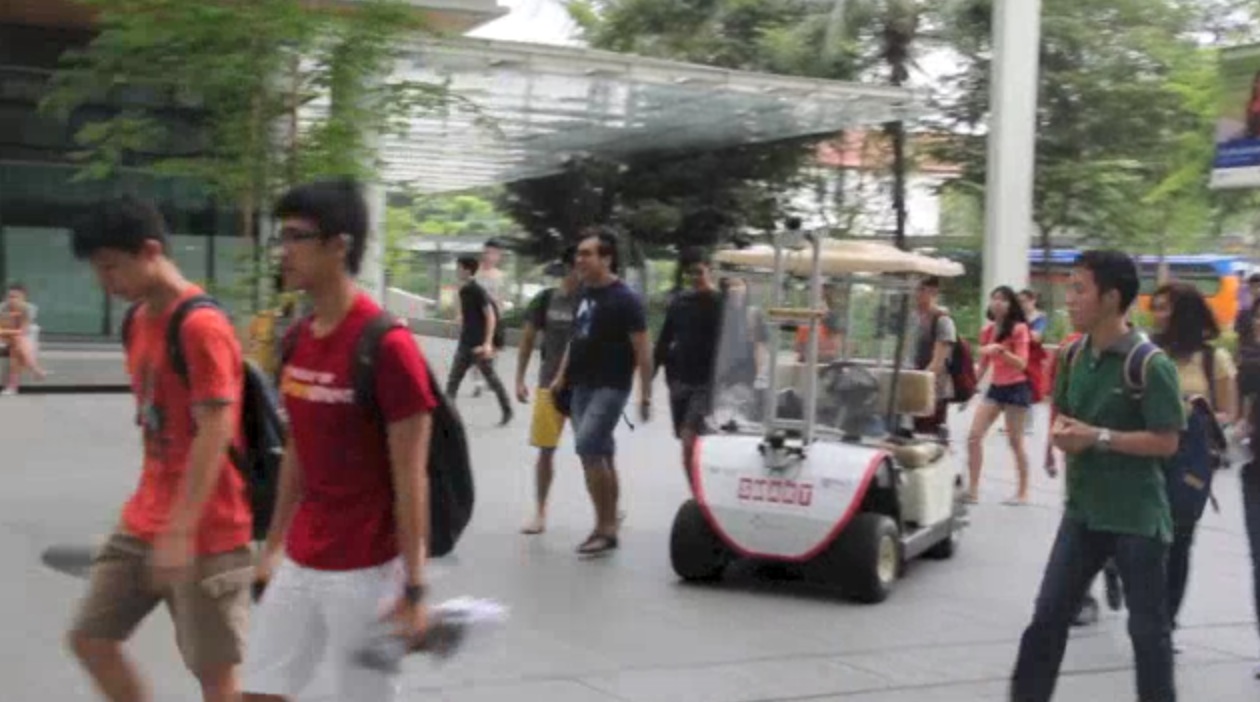}
\caption{DESPOT running in real time on an autonomous golf-cart among
  many pedestrians.}
	\label{fig:driving}
\end{center} \end{figure}

Experiments show that the anytime \despot algorithm is successful on
very large POMDPs with up to $10^{56}$ states. The algorithm is also capable of
handling complex dynamics. We implemented it on a robot
vehicle for intention-aware autonomous driving among many
pedestrians (Figure \ref{fig:driving}). It achieved real-time performance on this system \cite{BaiCai15}.  
In addition, the \despot algorithm is a core component of
our autonomous mine detection strategy that won the 2015 
Humanitarian Robotics and Automation Technology Challenge 
 \cite{madhavan2015} (\figref{fig:demining}).

\begin{figure}
\begin{center}
  \includegraphics[width=2.8in,height=1.9in]{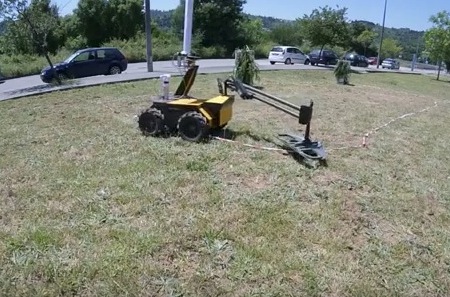}
  \quad
  \includegraphics[width=2.8in,height=1.9in]{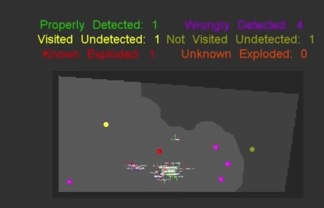}
\caption{Robot mine detection in  Humanitarian Robotics and Automation
  Technology Challenge 2015. Images used with permission of Lino Marques.}
	\label{fig:demining}
\end{center} \end{figure}

In the following, \secref{sec:background} reviews the background on POMDPs and
related work.  \secref{sec:despot} defines the \despot formally and presents
the theoretical analysis to motivate the use of a regularized objective
function for lookahead search.
\secref{sec:anytime} presents the anytime online POMDP algorithm, which
searches a \despot for a near-optimal policy.  \secref{sec:experiments}
presents experiments that evaluate our algorithm and compare it with the state
of the art.
\secref{sec:discussion} discusses the strengths and the limitations of this
work as well as opportunities for further work. We
conclude with a summary in \secref{sec:conclusions}.

\section{Background}
\label{sec:background}

We review the basics of online POMDP planning and related works in this
section.

\subsection{Online POMDP Planning}
\label{sec:pomdp}
A POMDP models an agent acting in a partially observable stochastic
environment.  It can be specified formally as a tuple $(S, A, Z, T, O, R)$,
where $S$ is a set of states, $A$ is a set of agent actions, and $Z$ is
a set of observations.
When  the agent takes action $a\in A$ in state $s\in S$, it moves to a new
state $s'\in S$ with
probability $T(s, a, s') = \pr(s' | s, a)$ and
receives observation $z \in Z$ with probability $O(s', a, z) = \pr(z | s',
a)$. It also receives a real-valued reward  $R(s, a)$.

A POMDP agent does not know the true state, but receives observations that
provide partial information on the state.  The agent thus maintains a
belief, represented as a probability distribution over~$S$. It
starts with an initial belief $b_0$.
At time $t$, it updates the belief 
according to Bayes' rule,  by incorporating information from
the action $a_t$ taken 
and the resulting observation $o_t$:
\begin{eq}\label{eq:1}
  b_t(s')
  =  \eta {O(s', a_{t}, z_{t})}  \sum_{s\in S} T(s, a_t, s') b_{t-1}(s), 
\end{eq}
where $\eta$ is a normalizing constant.
The belief $b_t=\tau(b_{t-1}, a_t,
z_t) = \tau(\tau(b_{t-2}, a_{t-1}, z_{t-1}) , a_t, z_t) = \cdots = \tau(\cdots
\tau(\tau(b_0, a_1,b_1), a_2, b_2), \ldots, a_t,z_t) $ is a sufficient
statistic that contains all the information from the history of actions and
observations $(a_1, z_1, a_2, z_2, \ldots, a_{t}, z_{t})$.

A \emph{policy} $\pi\colon \mathcal{B} \mapsto A$ is a mapping from the belief space
$\mathcal{B}$ to the action space $A$.  It prescribes an action $\pol(b) \in
A$ at the belief $b\in \mathcal{B}$.  For infinite-horizon POMDPs, the
\emph{value} of a policy \pol at a belief $b$ is the expected total discounted
reward that the agent receives by executing \pol:
\begin{eq}\label{eq:2}
V_\pol(b) = \ev\Bigl(\sum_{t=0}^{\infty} \gamma^t
R\bigl(s_t,\pol(b_t)\bigr)\;\bigl|\; b_0=b\Bigr).
\end{eq}
The constant $\gamma\in[0,1)$ is a discount factor, which expresses  preferences
for immediate rewards over future ones.

In online POMDP planning, the agent starts with an initial belief. At
each time step, it searches for an optimal action $a^*$ at the current belief
$b$.
The agent executes the action $a^*$ and receives a new
observation $z$.  It updates the belief  using \eqref{eq:1}.
The process then repeats.

To search for an optimal action $a^*$, one way is to construct a \emph{belief
  tree} (\figref{fig:despot}\subfig a), with the current belief $b_0$ as the
initial belief at the root of the tree. The agent performs lookahead search on
the tree for a policy \pol that maximizes the value $V_\pol(b_0)$ at $b_0$,
and sets $a^* = \pol(b_0)$.
Each node of the tree represents a belief.
To simplify the notation, 
we use the same notation $b$ to represent
both the node and the associated belief.  A node branches into $|A|$ action
edges, and  each action edge further 
branches into $|Z|$ observation edges. If a node and its child represent beliefs
$b$ and $b'$, respectively, then $b' = \tau(b,a,z)$ for some $a\in A$ and
$z\in Z$.

To obtain an approximately optimal policy, 
we may truncate the tree at a maximum depth~\height and search for the optimal
policy on the truncated tree.  At each leaf node, we simulate a user-specified
\emph{default policy} to obtain a lower-bound estimate on its optimal value.
A default policy, for example, can be a random policy or a heuristic.
At each internal node $b$, we apply Bellman's principle of optimality:
\begin{eq}\label{eq:3}
  V^*(b) = \max_{a\in A}\Bigl\{\sum_{s \in S}{b(s)R(s, a)} + \gamma\sum_{z \in
          Z}\pr(z|b, a)V^*\bigl(\tau(b, a, z)\bigr) \Bigr\},
\end{eq}
which computes the maximum value of action branches and the
average value of observation branches weighted by the observation
probabilities. We then perform a post-order traversal on the belief tree and
use \eqref{eq:3} to compute recursively the maximum value at
each node and  obtain the best action at the root
node $b_0$ for execution.

Suppose that at each internal node $b$ in a belief tree, we retain only one
action branch, which represents the chosen action at $b$, and remove all other
branches.  This transforms a belief tree into a \emph{policy tree}.  Each
internal node~$b$ of a policy tree has a single out-going action edge, which
specifies the action at $b$. Each leaf node is associated with a default
policy for execution at the node.  Our \despot algorithm uses this policy tree
representation.
We define the \emph{size} of such a policy as the
number of \emph{internal} policy tree nodes. A singleton policy tree
thus has size~$0$.

\subsection{Related Work}
\label{sec:related}

There are two main approaches to POMDP planning: offline policy computation
and online search.  In offline planning, the agent computes beforehand a
policy contingent upon all possible future outcomes and executes the computed
policy based on the observations received. One main advantage of offline planning
is fast policy execution, as the policy is precomputed.  Early work on
POMDP planning often takes the offline approach. See, \eg,
the work of \citeA{KaeLit98} and \citeA{zhang2001speeding}.
Although offline planning algorithms have made major
  progress in recent
years \cite{pingor03,spavla05,smisim05,KurHsu08}, they still face
significant difficulty in scaling up to very large POMDPs, as they must
plan for all beliefs and future contingencies.

Online planning interleaves planning with
plan execution.  At each time step, it plans locally and
chooses an optimal action for
the \emph{current belief} only, by performing lookahead
search in the neighborhood of the current belief. It then executes
the chosen action immediately.
Planning for the current belief is computationally attractive. First,
it is simpler to search for an optimal action at a single belief than
to do so for all beliefs, as offline policy computation does.
Second, it allows us to exploit local
structure to reduce the search space size.
One limitation of the online approach is the constraint on planning time.  As it
interleaves planning and plan execution, the plan must be ready for execution
within short time in some applications.

A recent survey lists three main ideas for online planning via belief tree
search~\cite{RosPin08}: heuristic search, branch-and-bound pruning, and Monte
Carlo sampling.  Heuristic search employs a heuristic to guide the
belief tree search~\cite{RosCha07,RosPin08}. This idea dates
back to the early work of \citeA{SatLav73} .
Branch-and-bound pruning maintains upper and lower bounds on the value at each
belief tree node and use them to prune suboptimal subtrees and improve
computational efficiency~\cite{PaqTob05}. This idea is also present in earlier
work on offline POMDP planning~\cite{SmiSim04}.
Monte Carlo sampling explores only a randomly sampled subset of observation
branches at each node of the belief tree~\cite{BerCas99,yoonfern08,KeaMan02,SilVen10}.  
Our \despot algorithm contains all three ideas, but is most closely
associated with Monte Carlo sampling. 
Below we examine some of the earlier Monte Carlo sampling
algorithms and \despot's connection with them.

The rollout algorithm~\cite{BerCas99} is an early example of Monte Carlo
sampling for planning under uncertainty. It is originally designed for Markov decision processes (MDPs), but
can be easily adapted to solve POMDPs as well.  It estimates the value of a
default heuristic policy by performing \nsam simulations and then
chooses the best action by \emph{one-step} lookahead search over the estimated
values. Although a one-step lookahead policy improves over the default
policy, it may be far from the optimum because of the very short, one-step
search horizon. 

Like the rollout algorithm, the hindsight optimization algorithm
(HO)~\cite{chonggi00,yoonfern08} is intended for MDPs, but can be adapted for
POMDPs.  While both HO and \despot sample $\nsam$ scenarios for planning, HO
builds one tree with $\O(|A|^\height)$ nodes for \emph{each} scenario,
independent of others, and thus $\nsam$ trees in total.  It searches each tree
for an optimal plan and averages the values of these $K$ optimal plans to
choose a best action.  HO and related algorithms have
been quite successful in recent international probabilistic planning
competitions.  However, HO plans for each scenario independently; it
optimizes an upper bound on the value of a POMDP and not  the true value
itself.  In contrast, \despot captures all $\nsam$ scenarios in a single tree
of $\O(|A|^\height \nsam)$ nodes and hedges against all $\nsam$ scenarios
simultaneously during the planning.  It converges to the true optimal value of the POMDP as
$\nsam$ grows.

The work of \citeA{KeaMan99} and that of
\citeA{NgJor00} use sampled scenarios as well, but for
offline POMDP policy computation.
They provide uniform convergence bounds in terms of the complexity of
the policy class under consideration, \eg, the Vapnik-Chervonenkis dimension.
In comparison, \despot uses sampled scenarios for online instead of offline
planning.   Furthermore, our competitive analysis of DESPOT
compares our computed policy against an optimal policy and produces a bound
that depends on the size of the optimal policy. 
The bound benefits from the existence of a small near-optimal policy and
naturally leads to a regularized objective function for online planning; in
contrast, the algorithms by \citeA{KeaMan99} and \citeA{NgJor00} do not
exploit the existence of good small policies within the class of policies
under consideration.

The sparse sampling (SS) algorithm \cite{KeaMan02} and the \despot
algorithm both construct sparse approximations to a belief tree. SS samples a
constant number $C$ of observation branches for each action.  A sparse
sampling tree contains $\O(|A|^DC^D)$ nodes, while a \despot contains
$\O(|A|^D K)$ nodes. 
Our analysis shows that $K$ can be much smaller than $C^D$ when a POMDP admits
a small near-optimal policy (Theorem~\ref{th:despot}).
In such cases, the \despot algorithm is computationally more efficient.

POMCP~\cite{SilVen10} performs Monte Carlo tree search (MCTS) on a belief
tree. It combines optimistic action exploration and random observation
sampling, and relies on the UCT algorithm \cite{KocSze06} to trade off
exploration and exploitation.  POMCP is simple to implement and is one of the
best in terms of practical performance on large POMDPs.
However, it can be misguided by the upper confidence bound (UCB) heuristic of
the UCT algorithm and be overly greedy.  Although it
converges to an optimal action in the limit, its worst-case running time is
extremely poor: $\Omega(\overset{\sss
  D-1}{\overbrace{\exp(\exp(\ldots\exp}}(1)\ldots)))$~\cite{CoqMun07}.

Among the Monte Carlo sampling algorithms, one unique feature of \despot is
the use of regularization to avoid overfitting to sampled scenarios:
it balances the estimated performance value of a policy and the policy size
during the online search, improving overall performance for suitable tasks.

Another important issue for online POMDP planning is belief representation.
Most online POMDP algorithms, including the Monte Carlo sampling algorithms,
explicitly represent the belief as a probability distribution over the state
space. This severely limits their scalability on POMDPs with very large state
spaces, because a single belief update can take time quadratic in the number
of states.  Notable exceptions are POMCP and DESPOT. Both represent the belief
as a set of sampled states and do not perform belief update over the entire
state space during the online search.

Online search and offline policy computation are complementary and can be
combined,  by using approximate or partial policies computed offline as
the default policies at the leaves of the search tree for online 
planning~\cite{BerCas99,GelSil07}, or as macro-actions to shorten the search
horizon~\cite{HeBru11}.

This paper extends our earlier work~\cite{SomYe13}. 
It provides an improved anytime online planning algorithm, an analysis of this
algorithm, and new experimental results.

\section{Determinized Sparse Partially Observable Trees} 
\label{sec:despot}

A \despot is a sparse approximation of a standard belief
tree.  While a standard belief tree captures the execution of all policies
under all possible scenarios, a \despot captures the execution of all policies
under a set of randomly sampled scenarios (\figref{fig:despot}\subfig b).
A DESPOT contains all the
action branches, but only the observation branches encountered under the
sampled scenarios.

We define \despot constructively by applying a \emph{deterministic simulative
  model} to all possible action sequences under $K$~sampled scenarios.
A scenario is an abstract simulation trajectory
with some start state $s_0$.  Formally, a \emph{scenario} for a belief $b$
is an infinite random sequence $\seq= (s_0, \seqt_1, \seqt_2, \ldots)$, in which the
start state $s_0$ is sampled according to $b$ and each $\seqt_i$ is a real
number sampled independently and uniformly from the range $[0,1]$.  The
deterministic simulative model is a function $g\colon S \times A \times
\textrm{R}\mapsto S \times Z$, such that if a random number~$\seqt$ is
distributed uniformly over $[0,1]$, then $(s', z') = g(s, a, \phi)$ is
distributed according to $\pr(s', z' | s, a)= T(s,a,s')O(s', a, z')$.  When 
simulating this model for an action sequence $(a_1, a_2, \ldots)$ under a
scenario $(s_0, \seqt_1, \seqt_2, \ldots)$, we get  a simulation
trajectory $(s_0,a_1,s_1,z_1,a_2,s_2,z_2,\ldots)$, where $(s_t,z_t) =
g(s_{t-1}, a_t, \seqt_t)$ for $t=1,2, \ldots$\ .  The simulation trajectory
traces out a path $(a_1,z_1,a_2,z_2,\ldots)$ from
the root of the standard
belief tree. We add all the nodes and edges on this path to the \despot \dtree
being constructed. Each
 node \node of \dtree  contains a set $\seqset_\node$  of all scenarios
that it encounters.
We insert the scenario
$(s_0, \seqt_0, \seqt_1, \ldots)$ into the set $\seqset_{b_0}$ at the root
$b_0$ and insert the scenario
$(s_t, \seqt_{t+1}, \seqt_{t+2}, \ldots)$ into the set $\seqset_{b_t}$ at the
belief node $b_t$ reached at the end of the subpath $(a_1,z_1,a_2,z_2,\ldots,
a_t,z_t)$, for $t=1,2, \ldots$\ .  Repeating this process for \emph{every} action
sequence under every sampled scenario completes the construction of 
\dtree.

In summary, a \despot is a randomly sampled subtree of a standard
belief tree.  It is completely determined by the set of $K$ random sequences
sampled a priori, hence the name \emph{Determinized Sparse Partially Observable Tree}.
Each \despot node $b$ represents a belief and contains a
set $\seqset_{b}$ of scenarios.
The start states of the scenarios in $\seqset_\node$ form a particle set
that represents $b$ approximately.
While a standard belief tree of height \height has $\O(|A|^\height
|Z|^\height)$ nodes, a corresponding \despot has $\O(|A|^\height
\nsam)$ nodes for $|A|\geq 2$, because of reduced observation branching under the
sampled scenarios. 

It is possible to search for near-optimal policies using a \despot instead of a standard belief tree.
The \emph{empirical value} $\hat{V}_{\pol}(b)$ of a policy \pol under the
sampled scenarios encoded in a \despot
is the average total discounted reward obtained by
simulating the policy under each scenario.
Formally, let $V_{\pol, \seq}$ be the
total discounted reward of the trajectory obtained by simulating \pol under a
scenario $\seq\in \seqset_b$ for some node $b$ in a \despot, then
\begin{dm}
\hat{V}_{\pol}(b) = \sum_{\seq \in \seqset_b}
                           {V_{\pol,  \seq} \over  |\seqset_b|},
\end{dm}
where $|\seqset_b|$ is the number of scenarios in $\seqset_b$.  Since
$\hat{V}_{\pol}(b)$ converges to $V_\pol(b)$ almost surely as
$K\rightarrow \infty$, the problem of finding an optimal policy at $b$ can be
approximated as that of doing so under the sampled scenarios.  One
concern, however,  is \emph{overfitting}: a policy optimized for finitely many sampled
scenarios may not be optimal in general, as many scenarios are not sampled.
To control overfitting, we regularize the empirical value of a policy by
adding a term that penalizes large policy size. We now provide theoretical
analysis to justify this approach.

Our first result bounds the error of estimating the values
of \emph{all}  policies derived from DESPOTs of a given size.
The result implies that a \despot constructed with a small
number of scenarios is sufficient for approximate policy evaluation.  The
second result shows that by optimizing this bound, which is equivalent to
maximizing a regularized value function, we obtain a policy that is
competitive with the best small policy. 

Formally, a \despot \emph{policy} $\pol$ is a policy tree
derived from a \despot \dtree: $\pol$ contains the same
root as the \despot \dtree, but only one action branch at each internal node.  To execute
$\pol$, an agent starts at the root of $\pol$.  At
each time step, it takes the action specified by the action edge at the
node. Upon receiving the resulting observation, it follows the corresponding
observation edge to the next node.  The agent may encounter an observation not
present in  $\pol$, as $\pol$ contains only the observation
branches under the sampled scenarios.  In this case, the agent follows
a default policy from then on. Similarly, it follows the default policy
when reaching a leaf node of $\pol$.
Consider the set $\polclass_{b_0, \height,\nsam}$,  which contains all
\despot policies derived from {\despot}s of height $\height$,  constructed with
\emph{all} possible $K$ sampled scenarios for a belief $b_0$.
We now bound the error on the estimated value of an arbitrary \despot policy
in $\polclass_{b_0, \height,\nsam}$. 
To simplify the presentation, we assume without
loss of generality that all rewards are non-negative and are bounded by \rmax. For models with bounded negative
rewards, we can shift all rewards by a constant to make them
non-negative. The shift does not affect the optimal policy.

\begin{theorem} \label{th:uniform}
  For any given constants $\tau, \alpha \in (0,1)$, any 
  belief $b_0$, and any positive integers $D$ and $K$, \emph{every} \despot
  policy tree $\pol
  \in \polclass_{b_0, \height,\nsam}$ satisfies
\begin{eq}\label{eq:4}
V_{\pi}(b_0) \geq
  \frac{1-\alpha}{1+\alpha}\hat{V}_{\pi}(b_0)
	- \frac{\rmax}{(1+\alpha)(1-\gamma)}\cdot\frac{\ln(4 / \tau) +
       |\pi|\ln \bigl(KD |A| |Z|\bigr)}{\alpha K}
   \end{eq}
with probability at least $1 - \tau$,   
where  $\hat{V}_{\pi}(b_0)$ is the estimated value of $\pi$
under a set of $\nsam$ scenarios randomly sampled according to $b_0$.
\end{theorem}
All proofs are presented in the appendix.
Intuitively, the  result  says that all
\despot policies in $ \polclass_{b_0, \height,\nsam}$
satisfy the bound given in \eqref{eq:4}, with high probability.
The bound holds for any constant
$\alpha \in (0,1)$, which is  a parameter that can be tuned to tighten
the bound.
A smaller $\alpha$ value reduces the approximation error in the first term on
the right-hand side of \eqref{eq:4}, but increases the additive error in the second term.
The additive error depends on the size of \pol.
It also grows logarithmically with $|A|$ and $|Z|$. 
The estimation thus scales well with large action and observation
spaces. We can make this estimation error
arbitrarily small by choosing a suitable 
number of sampled scenarios, \nsam.

The next theorem states that we can obtain a near-optimal
policy ${\hat \pol}$ by maximizing
the RHS of \eqref{eq:4}, which accounts for both the estimated performance
and the size of a policy.
\begin{theorem}\label{th:despot}
  Let $\pol$ be an arbitrary policy at a belief $b_0$.  Let
  $\Pi_\mathcal{D}$ be the set of policies derived from a
\despot~$\mathcal{D}$ that has height \height and is constructed with
$\nsam$ scenarios sampled randomly according to $b_0$.  For any given
constants  $\tau,
\alpha \in (0,1)$, if
\begin{eq}\label{eq:5}
    \small
{\hat \pol} = \argmax_{\pol' \in \Pi_\mathcal{D}} \biggl\{   \frac{1-\alpha}{1+\alpha}\hat V_{\pi'}(b_0) -
    \frac{R_{\max}}{(1+\alpha)(1-\gamma)}\cdot\frac{|\pi'|\ln\bigl(KD|A||Z|\bigr)}{\alpha
        K}\biggr\},
  \end{eq}
then
\begin{dm}
  \resizebox{0.95\hsize}{!}{$
  V_{\hat \pi}(b_0) \geq \frac{1-\alpha}{1+\alpha}V_\pi(b_0) -
  \frac{R_{\max}}{(1+\alpha)(1-\gamma)}\biggl(\frac{\ln(8 / \tau) +
      |\pi|\ln\bigl(KD|A||Z|\bigr)}{\alpha K} + (1-\alpha)\Bigl(\sqrt{\frac{2\ln
          (2/\tau)}{K}}+\gamma^{\sss D}\Bigr)\biggr)
  $}
\end{dm}
with probability at least $1 - \tau$. 
\end{theorem}
Theorem~\ref{th:despot} bounds the performance of $\hat \pol$, the
policy maximizing~\eqref{eq:5} in the \despot, in terms of the
performance of another policy $\pol$. Any policy
can be used as the policy $\pol$ for comparison in the
competitive analysis.  Hence, the
performance of $\hat \pol$ can be compared to the performance of an optimal
policy. If the optimal policy has small representation size, the
approximation error of $\hat \pol$ is correspondingly small. The
performance of $\hat \pol$ is also robust.  If the optimal policy has
large size, but is well approximated by a small policy \pol of size
$|\pol|$, then we can obtain $\hat \pol$ with small approximation
error, by choosing $K$ to be $\O(|\pol|\ln (|\pol| |A| |Z|))$.
Since a \despot has size $\O(|A|^\height \nsam)$,
the choice of $K$ allows us to trade off
computation cost  and approximation accuracy.

The objective function in~\eqref{eq:5} has the form 
\begin{eq}\label{eq:6}
\hat{V}_{\pol}(b_0) - \lambda |\pol|
\end{eq}
for some $\lambda \ge 0$, 
similar to that of regularized utility functions in many machine
learning algorithms.
This motivates the regularized objective function for our online planning
algorithm described in the next section.

\section{Online Planning with {\despot}s}
\label{sec:anytime}
Following the standard online planning framework (\secref{sec:pomdp}), our
algorithm iterates over two main steps: action selection and belief update.  For
belief update, we use a standard particle filtering method,
 sequential importance
resampling (SIR)~\cite{GorSal93}. 

We now present two action selection methods. In Section~\ref{sec:dp}, we
describe a conceptually simple dynamic programming method that
constructs a \despot fully before
finding the optimal action. For very large POMDPs,
constructing the \despot fully is not practical.
In Sections~\ref{sec:despot-approximation} to \ref{sec:lower}, we
describe an anytime \despot algorithm that performs anytime heuristic
search.  The anytime algorithm constructs a \despot
incrementally under the guidance of a heuristic and scales up to
very large POMDPs in practice. In \secref{subsec:analysis},
we show that the algorithm converges to an optimal policy when the
heuristic is admissible and  that the performance of the
algorithm degrades gracefully even when the heuristic is not
admissible.

\subsection{Dynamic Programming}
\label{sec:dp}
We construct a fixed \despot \dtree with $K$ randomly sampled scenarios and
want to derive from \dtree a policy that maximizes the regularized empirical
value \eqref{eq:6} under the sampled scenarios:
\begin{dm}
  \max_{\pi\in \Pi_\dtree} \Bigl\{\hat{V}_{\pol}(b_0) - \lambda |\pol|\Bigr\},
\end{dm}
where $b_0$  is the current belief,  at the root of \dtree.
Recall that a DESPOT policy is represented as a policy tree.  For each node  $b$ of \pol, we define the \emph{regularized
weighted discounted utility} (RWDU):
\begin{equation}
\nu_{\pol}(b) = \frac{|\seqset_{b}|}{K} \gamma^{\depth(b)} \hat{V}_{\pol_b}(b) -
\lambda |\pol_b|, \label{eqn:rwdu}
\end{equation}
where $|\seqset_{b}|$ is the number of scenarios passing through node
$b$,  $\gamma$ is the discount factor, $\depth(b)$ is the depth of $b$ in
the policy tree \pol, $\pol_b$ is the subtree rooted at $b$, and
$|\pol_b|$ is the size of $\pol_b$.
The ratio ${|\seqset_{b}|}/{K}$ is an empirical estimate of the
probability of reaching $b$.
Clearly,  $ \nu_{\pol}(b_0)=\hat{V}_{\pol}(b_0) - \lambda |\pol|$, which we
want to optimize.

For every node $b$
of $\dtree$, define $\nu^*(b)$ as the maximum RWDU of $b$ over all policies in
$ \Pi_\dtree$.
Assume that \dtree has finite depth.
We now present  a dynamic programming procedure that computes
$\nu^*(b_0)$  recursively from bottom
up. At a leaf node $b$ of \dtree, we simulate a default policy $\pol_0$
under the sampled scenarios. According to our definition in \secref{sec:pomdp},
$|\pol_0| = 0$. Thus,
\begin{eq}
	\nu^*(b) = \frac{|\seqset_{b}|}{K}\gamma^{\depth(b)} \hat{V}_{\pi_0}(b).
\end{eq}
At each internal node $b$, let $\tau(b,a,z)$ be the child of $b$ following
the action branch~$a$ and the observation branch $z$ at $b$. Then,  
\begin{align}
	\nu^*(b) &= \max\biggl\{
		\frac{|\seqset_{b}|}{K} \gamma^{\depth(b)} \hat{V}_{\pol_0}(b),\; 
			\max_{a\in A} 
			\Bigr\{ \rho(b, a) + \sum_{z \in Z_{b,a}} \nu^*(\tau(b,a,z)) \Bigl\}
		\biggr\},\label{eqn:best-reg-pol}
\end{align}
where 
\begin{align*}
    \rho(b, a) &=  {1\over \nsam}\sum_{\seq \in \seqset_b}{\gamma^{\depth(b)}}
		 R(s_\seq, a) - \lambda,
\end{align*}
the state $s_\seq$ is the start state of the scenario $\seq$, and $Z_{b,a}$ is
the set of observations following the action branch~$a$ at the node $b$.
The outer maximization in \eqref{eqn:best-reg-pol} chooses between
executing the default policy or expanding the subtree at $b$.
As the RWDU contains a regularization term, the latter  is beneficial
only if the subtree at $b$ is relatively small. This effectively prevents
expanding a large subtree, when there are an insufficient number of sampled
scenarios to estimate the value of a policy at $b$ accurately. 
The inner in \eqref{eqn:best-reg-pol}
maximization chooses among the different actions available.
When the algorithm terminates,
the maximizer at the root $b_0$ of \dtree gives the best action at~$b_0$.

If \dtree has unbounded depth, it is sufficient to truncate \dtree to a depth
of $ \lceil{R_\mathrm{max}} /{ \lambda (1-\gamma)}\rceil + 1$ and run the
above algorithm, provided that $\lambda > 0$.  The reason is that an optimal
regularized policy $\hat \pol$ cannot include the truncated nodes of \dtree.
Otherwise, $\hat \pol$ has size at least
$ \lceil{R_\mathrm{max}} /{ \lambda (1-\gamma)}\rceil + 1$ and thus
RWDU $\nu_{\hat \pol}(b_0) < 0$.
Since the default policy $\pi_0$ has RWDU $\nu_{\pol_0}(b_0) \geq 0$, 
$\pol_0$ is then better than $\hat \pol$, a contradiction.

This dynamic programming algorithm runs in time linear in the number of nodes
in  \dtree.
We first simulate the deterministic model to construct the tree, then do a
bottom-up dynamic programming to initialize $\hat{V}_{\pi_0}(b)$, and finally
compute $\nu^*(b)$ using Equation~\eqref{eqn:best-reg-pol}.
Each step takes time linear in the number of nodes in \dtree given previous
steps are done, and thus the total running time is $\O(|A|^\height\nsam)$.

\subsection{Anytime Heuristic Search} \label{sec:despot-approximation}

The bottom-up dynamic programming algorithm in the previous section constructs
the full \despot \dtree in advance.
This is generally not practical because there are exponentially many nodes.
To scale up, we now present an anytime forward search
algorithm\footnote{This algorithm differs from an earlier version
  \cite{SomYe13} in a subtle, but important way.  The new algorithm optimizes
  the RWDU directly by interleaving incremental \despot construction and
  backup.  The earlier one performs incremental \despot construction without
  regularization and then optimizes the RWDU over the constructed \despot.  As
  a result, the new algorithm is guaranteed to converge to an optimal
  regularized policy derived from the full \despot, while the earlier one is
  not.}  that avoids constructing the \despot fully in advance.  It selects
the action by incrementally constructing a \despot $\dtree$ rooted at the
current belief $b_0$, using heuristic search~\cite{SmiSim04,KurHsu08}, and
approximating the optimal RWDU $\nu^*(b_0)$.  We describe the main components
of the algorithm below. The complete pseudocode is
given in Appendix~\ref{sec:pseudocode}.

To guide the heuristic search, 
we maintain a lower bound
$\lowernu(b)$ and an upper bound $\uppernu(b)$ on the optimal RWDU at each
node $b$ of \dtree,  so that
$\lowernu(b) \leq \nu^*(b) \leq \uppernu(b)$.
To prune the search tree, 
 we additionally maintain an
upper bound $\upperv(b)$ on the empirical value $\hat{V}^*(b)$ of the optimal
regularized policy so that $\upperv(b) \geq
\hat{V}^*(b) $ and compute an initial lower bound $\lowerv_0(b)$ with
$\lowerv_0(b) \leq \hat{V}^*(b)$. In particular,
we use 
$\lowerv_0(b) = \hat{V}_{\pol_0}(b)$ for the default policy $\pol_0$ at $b$.

\begin{algorithm}[t]
\caption{$\textsc{BuildDespot}(\bvar{b}_0)$} \label{alg:builddespot}
\begin{algorithmic}[1]
\STATE Sample randomly a set $\seqset_{{b}_0}$ of $\nsam$  scenarios 
from the current belief ${b}_0$. 
\STATE Create a new \despot \dtree with a single node ${b}_0$ as the root.
\STATE Initialize $\upperv({b}_0), 
\lowerv_0({b}_0), \uppernu({b}_0), \text{ and } \lowernu({b}_0) $.
\STATE $\gap({b}_0) \gets \uppernu({b}_0) - \lowernu({b}_0)$. 
\WHILE {$\gap({b}_0) >  \gap_0$
          and the total running time is less than \tmax}
\STATE $b \gets \textsc{Explore}(\dtree, {b}_0)$.
\STATE $\textsc{Backup}(\dtree, b)$.
\STATE $\gap({b}_0) \gets \uppernu({b}_0) - \lowernu({b}_0)$. 
\ENDWHILE
\RETURN $\lowernu$
\end{algorithmic}
\end{algorithm}

\algref{alg:builddespot} provides a high-level sketch of the algorithm. We
construct and search a  \despot
\dtree incrementally, using $\nsam$ sampled scenarios (line~1).
Initially, \dtree contains only a single root node with belief
${b}_0$ and
the associated initial upper and lower bounds (lines 2--3).  The
algorithm makes a series of explorations to expand \dtree and reduce
the gap between the bounds $\uppernu(b_0)$ and $\lowernu({b}_0)$ at the
root node ${b}_0$ of \dtree. Each exploration follows a
heuristic and traverses a promising path from the root of \dtree to add new
nodes to \dtree (line~6). 
Specifically, it keeps on choosing and expanding a promising leaf node and
adds its child nodes into \dtree until current leaf node is not heuristically
promising. The algorithm then traces the path back to the root and
performs backup on the upper and lower bounds at each node along the way, using
Bellman's principle (line~7).  The explorations continue, until the gap
between the bounds $\uppernu({b}_0)$ and $\lowernu({b}_0)$ reaches a
target level $\gap_0 \geq 0$ or the allocated online planning time
runs out (line~5).

\subsubsection{Forward Exploration} \label{sec:trial}

\begin{algorithm}[t]
\caption{$\textsc{Explore}(\dtree, b)$} \label{alg:explore}
\begin{algorithmic}[1]
\WHILE {$\depth(b) \leq D$,  $\wdeu(b) > 0$, and \textsc{Prune}($\dtree$, $b$) =
  {\small FALSE}}
\IF {$b$ is a leaf node in $\dtree$}
	\STATE Expand $b$ one level deeper.  Insert each new child $b'$ of $b$
        into $\dtree$.
        Initialize $\upperv(b')$,
        $\lowerv_0(b')$,  $\uppernu(b')$,   and  $\lowernu(b')$.
\ENDIF
\STATE $a^* \gets \argmax_{a \in A} {\uppernu(b, a)}$.
\STATE $z^* \gets \argmax_{z \in Z_{b,a^*}} \wdeu(\tau(b,a^*,z))$.
\STATE $b \gets \tau(b, a^*, z^*)$.
\ENDWHILE
\IF {$\depth(b) > D$}
	\STATE \textsc{MakeDefault}($b$).
\ENDIF
\RETURN $b$.
\end{algorithmic}
\end{algorithm}
\begin{algorithm}[t]
\caption{\textsc{MakeDefault}($b$)} \label{alg:makedefault}
\begin{algorithmic}[t]
	\STATE $\upperv(b) \gets \lowerv_0(b)$.
	\STATE $\uppernu(b) \gets \lowernu_0(b)$.
	\STATE $\lowernu(b) \gets \lowernu_0(b)$.
\end{algorithmic}
\end{algorithm}

Let $\gap(b)=\uppernu(b) - \lowernu(b) $ denote the gap between
the upper and lower RWDU bounds at a node $b$. Each exploration aims to reduce the
\emph{current} gap $\gap(b_0)$ at the root $b_0$ to $\xi \gap(b_0)$ for some
given constant $0<\xi<1$ (Algorithm~\ref{alg:explore}). An exploration
starts at the root $b_0$. At each node $b$ along the exploration path,
we choose the action branch optimistically according to the upper bound $\uppernu(b)$:
\begin{eq}  \label{eq:8}
 a^* =  \argmax_{a\in A}  \uppernu(b,a) = \argmax_{a\in A} \Bigl\{\rho(b, a) +
 \sum_{z \in Z_{b, a}}   \uppernu(b') \Bigr\},
\end{eq}
where
$b'=\tau(b,a,z)$ is the child of $b$
following the action branch $a$ and the observation branch $z$ at $b$.
We then choose the observation branch $z$ that   
leads to a child node  $b'=\tau(b,a^*,z)$ maximizing
the \emph{excess uncertainty} $\wdeu(b')$ at $b'$:
\begin{eq}
  \label{eq:9}
z^* =  \argmax_{z \in Z_{b,a^*}} \wdeu(b')
    = \argmax_{z \in Z_{b,a^*}} \Bigl\{ \gap(b') - \frac{|\seqset_{b'}|}{K}
    \cdot \xi \, \gap(b_0)\Bigr\}.
\end{eq}
Intuitively, the excess uncertainty $\wdeu(b')$ measures the difference
between the current gap at $b'$ and the ``expected'' gap at $b'$ if the target
gap $\xi\, \gap(b_0)$ at $b_0$ is satisfied.
Our exploration strategy seeks to reduce the excess uncertainty in a
greedy manner. See Lemma~\ref{lem:forward} in \secref{subsec:analysis}, the
work of \citeA{smisim05} and the work of \citeA{KurHsu08} for justifications of this strategy.

If the exploration encounters a leaf node $b$, we expand $b$ by creating a child
$b'$ of $b$ for each action $a \in A$ and each observation encountered under a
scenario $\seq \in \seqset_{b}$.  For each new child $b'$, we need to
compute the initial bounds $\uppernu_0(b')$, $\lowernu_0(b')$, $\upperv_0(b')$, and
$\lowerv_0(b')$.  The RWDU bounds $\uppernu_0(b')$ and $\lowernu_0(b')$ can be
expressed in terms of the empirical value bounds $\upperv_0(b')$ and
$\lowerv_0(b')$, respectively.
Applying the default policy $\pol_0$ at $b'$  and 
using the definition of RWDU in
\eqref{eqn:rwdu}, we have 
\begin{dm}
\lowernu_0(b') =  \nu_{\pol_0}(b') = \frac{|\seqset_{b'}|}{K}
\gamma^{\depth(b')} \lowerv_0(b'), 
\end{dm}
as $|\pol_0|=0$.
For the initial upper bound $\uppernu_0(b')$, there are two cases.
If the policy for maximizing the RWDU at $b'$ is the default policy,
then we can set $\uppernu_0(b') = \lowernu_0(b')$. Otherwise, the optimal policy has size at
least $1$, and it follows from \eqref{eqn:rwdu} that $\uppernu_0(b') =
{|\seqset_{b'}| \over  K} \gamma^{\depth(b')} \upperv_0(b') - \lambda$ is an upper bound. So we have
\begin{align*}
\uppernu_0(b') & = \max\Bigl\{\lowernu_0(b'), \frac{|\seqset_{b'}|}{K} \gamma^{\depth(b')} \upperv_0(b') - \lambda\Bigr\}.
\end{align*}
There are various way to  construct the initial empirical value bounds
$U_0$ and $L_0$. We defer the discussion
to Sections~\ref{sec:upper} and \ref{sec:lower}.

Note that a node at a depth more than $D$ executes the default policy, and the
bounds are set accordingly using the \textsc{MakeDefault} procedure (Algorithm~\ref{alg:makedefault}).
We explain the termination conditions for exploration next.

\subsubsection{Termination of Exploration and Pruning} \label{sec:terminate}
We terminate the exploration at a node $b$  under three conditions
(Algorithm~\ref{alg:explore}, line~1).
First, $\depth(b) > D$, \ie, the maximum tree height is exceeded. Second, 
$\wdeu(b) < 0$, indicating that the expected gap at $b$ is reached and
further exploration from $b$ onwards  may be unprofitable.
Finally, $b$ is \emph{blocked} by an ancestor node~$b'$:
\begin{eq}
 \label{eqn:blocker}
\frac{|\seqset_{b'}|}{K} \gamma^{\depth(b')} (\upperv(b') - \lowerv_0(b'))
\le \lambda \cdot  \len(b', b),
\end{eq}
where $\len(b',b)$ is the number of nodes on the path from $b'$ to
$b$.
The
intuition behind this last condition is that there is insufficient number of
sampled scenarios at the ancestor node $b'$.  Further expanding~$b$ and thus
enlarging the policy subtree at $b'$ may cause overfitting and reduce the
regularized utility at $b'$.  
  We thus prune the search by applying the default policy
  at $b$ and setting the bounds accordingly by calling
\textsc{MakeDefault}.
More details are available in
Lemma~\ref{lem:blocking}, which derives the condition \eqref{eqn:blocker} and
proves that pruning the search does not compromise the optimality of the
algorithm.

The pruning process continues backwards from $b$ towards the root of
\dtree (Algorithm~\ref{alg:prune}). As the bounds are updated, new nodes may satisfy the condition for
pruning and are pruned away.

\begin{algorithm}[h]
\caption{\textsc{Prune}($\dtree, b$)} \label{alg:prune}
\begin{algorithmic}[1]
\STATE $\textsc{Blocked} \gets \text{\small FALSE}$.
\FOR{each node $x$ on the path from $b$ to the root of \dtree}
	\IF{$x$ is blocked by any ancestor node in \dtree}
		\STATE \textsc{MakeDefault}($x$).
		\STATE \textsc{Backup}(\dtree, $x$).
                \STATE $\textsc{Blocked} \gets \text{\small TRUE}$.
	\ELSE
		\STATE \textbf{break}
	\ENDIF
\ENDFOR
\RETURN \textsc{Blocked}
\end{algorithmic}
\end{algorithm}

\subsubsection{Backup} \label{sec:backup}

When the exploration terminates, we trace the path back to the
root and perform \emph{backup} on the bounds at each node $b$ along the way,
using Bellman's principle (Algorithm~\ref{alg:backup}):
\begin{align*}
\uppernu(b) &= \max \Bigl\{\lowernu_0(b), 
	\max_{a\in A} \Big\{\rho(b, a) + \sum_{z \in Z_{b,a}} \uppernu(b')\Bigr\} \Bigr\}, \\
\lowernu(b) &= \max \Bigl\{\lowernu_0(b), 
	\max_{a\in A} \Big\{\rho(b, a) + \sum_{z \in Z_{b,a}}  \lowernu(b')\Bigr\} \Bigr\}, \\ 
\upperv(b) &= \max_{a \in A}  \biggl\{ \frac{1}{|\seqset_{b}|}\sum_{\seq \in
\seqset_{b}} R(s_\seq, a) + \gamma\sum_{z \in Z_{b,a}} \frac{|\seqset_{b'}|}{|\seqset_b|} \upperv(b') \biggr\},
\end{align*}
where $b'$ is a child of $b$ with  $b'=\tau(b,a,z)$. 

\begin{algorithm}[t]
\caption{\textsc{Backup}($\dtree, b$)} \label{alg:backup}
\begin{algorithmic}[1]
\FOR{each node $x$ on the path from $b$ to the root of \dtree}
\STATE  Perform backup on $\uppernu(x) $, $\lowernu(x) $, and 
  $\upperv(x)$. 
\ENDFOR
\end{algorithmic}
\end{algorithm}

\subsubsection{Running Time}
Suppose that the anytime search algorithm invokes \textsc{Explore} $N$ times.
\textsc{Explore} traverses a path from the root to a leaf node of a \despot
\dtree, visiting at most  $D + K-1$ nodes along the way because a path has
at most $D$ nodes, and at most $K-1$ nodes not on the path can be added due to
node expansions
(\algref{alg:explore}). At each node, the following steps dominate the
running time.  Checking the condition for pruning (line~1) takes time
$\O(D^2)$ in total and thus $\O(D)$ per node.
Adding a new node to \dtree and initializing the bounds (lines~3 and~8) take
time $\O(I)$ (assuming $I$ is an upper bound of the cost).
Choosing the action branch (line~4) takes time $\O(|A|)$.  Choosing the
observation branch (line~5) takes time $\min\{|Z|, K\}\in \O(K)$, which is loose
because only the sampled observation branches are involved.
Thus, the running time at each node is $\O(D+I+|A|+K)$, and
the total running time is $\O\bigl(N (D+K) (D+I+|A|+K) \bigr)$.

The anytime search algorithm constructs a partial \despot 
with at most $N (D + K)$ nodes, while the dynamic programming algorithm
(\secref{sec:dp}) constructs a \despot fully with $\O(|A|^D K)$ nodes.
While the bounds are not directly comparable, 
$N (D + K)$ is typically much smaller than $|A|^D K $ in  many
practical settings.  This is the main difference between the two
algorithms.
The anytime search algorithm takes slightly more time at each node in order to
prune the \despot. The trade-off is overall beneficial for reduced \despot size. 

\subsection{Initial Upper Bounds}\label{sec:upper}
For illustration purposes,
we discuss several methods for constructing the initial upper
bound $\upperv_0(b)$ at a \despot node $b$. There are, of course,  many
alternatives. The flexibility in constructing upper and lower bounds for
improved performance is one
strength of \despot. 

The simplest one is  the
\emph{uninformed  bound}
\begin{eq}\label{eq:uninformed}
\upperv_0(b)  = {R_{\max} / (1-\gamma)}.
\end{eq}
While this bound is loose, it is easy to compute and may yield
good results when combined with suitable lower bounds.

\emph{Hindsight optimization} \cite{yoonfern08} provides a principled
method to construct an upper bound algorithmically.
Given a fixed scenario $\seq=(s_0, \phi_{1}, \phi_{2}, \ldots)$,
computing an upper bound on the total discounted reward achieved by
any arbitrary policy is a deterministic planning problem. When the
state space is sufficiently small, we can solve it by dynamic
programming on a trellis of $D$ time slices.  Trellis nodes represent
states, and edges represent actions at each time step.  Let $u(t, s)$
be the maximum total reward on the scenario $(s_t, \phi_{t+1},
\phi_{t+2}, \ldots, \phi_{\sss D})$ at state $s\in S$ and time step
$t$.  For all $s\in S$ and $t=0, 1, \ldots, D-1$, we set
\begin{dm}
  u(D, s) = \rmax / (1 - \gamma)
\end{dm}
and
\begin{dm}
u(t, s) = \max_{a\in A} \bigl\{R(s, a) + \gamma\, u(t+1, s' )\bigr\}, 
\end{dm}
where $s'$ is the new state given by the deterministic simulative model
$g(s, a, \phi_{t+1})$.
Then $u(0, s_0)$ gives the upper bound under $\seq=(s_0, \phi_1, \phi_2,
\dots)$.
We repeat this procedure for every $\seq\in \seqset_{b}$, and set
\begin{eq}\label{eq:ho}
U_0(b) =  {1 \over |\seqset_b|}\sum_{\seq \in \seqset_b}  u(0, s_\seq),
\end{eq}
where $s_\seq$ is the start state of \seq.
For a set of $K$ scenarios, this bound
can be pre-computed in $\O(K|S||A|D)$ time
and stored, before online planning starts.
To tighten this bound further, we may exploit domain-specific knowledge or
other techniques to 
initialize  $u(D, s)$ either exactly or heuristically, instead of using the
uninformed bound.
If $u(D, s)$ is a true upper bound on the total discounted reward,
then the resulting  $\upperv_0(b)$ is also a true upper bound. 

Hindsight optimization may be too expensive to compute when the state space is
large. Instead, we may do approximate hindsight optimization, by constructing a domain-specific heuristic upper bound $u_{\sss
  \mathrm{H}}(s_0)$ on the total discounted reward for each scenario
$\seq=(s_0, \phi_{1}, \phi_{2}, \ldots)$ and then use the average
\begin{eq}\label{eq:specific}
U_0(b) =  {1 \over |\seqset_b|}\sum_{\seq \in \seqset_b}  u_{\sss \mathrm{H}}(s_\seq)
\end{eq}
as an upper bound. This upper bound depends on the state only and is often
simpler to compute. Domain dependent knowledge can be used in constructing this bound -- this is often crucial in practical problems.
In addition, $u_{\sss \mathrm{H}}(s_\seq)$ need not be a true upper
bound. An approximation suffices.  Our analysis in
Section~\ref{subsec:analysis} shows that the \despot algorithm is
robust against upper bound approximation error, and the performance of
the algorithm degrades gracefully. We call this class of upper bounds,
\emph{approximate hindsight optimization}.

A useful approximate hindsight optimization bound can be obtained by assuming that the states are fully observable, converting the POMDP into a
corresponding MDP, and solving for its optimal value function
$V_{\sss \mathrm{MDP}}$. The expected value $V(b) = \sum_{s\in S} b(s)
V_{\sss \mathrm{MDP}}(s)$ is an
upper bound on the optimal value 
$V^*(b)$ for the POMDP, and  
\begin{eq}\label{eq:mdpbound}
\upperv_0(b) =
\frac{1}{|\seqset_{b}|} \sum_{\seq \in \seqset_{b}} V_{\sss \mathrm{MDP}}(s_\seq)
\end{eq}
approximates $V(b)$ by taking the average over the start states of the
sampled scenarios.  Like the domain-specific heuristic bound, the MDP
bound \eqref{eq:mdpbound} is in general not a true upper bound of the
RWDU, but only an approximation, because the MDP is not restricted to
the set of sampled scenarios.  It takes $\O(|S|^2|A|D)$ time to solve
the MDP using value iteration, but the running time can be
significantly faster for MDPs with sparse transitions.

\subsection{Initial Lower Bounds and Default Policies}\label{sec:lower}

The \despot algorithm requires a default policy $\pi_0$. The simplest
default policy is a \emph{fixed-action policy} with the highest
expected total discounted reward \cite{SmiSim04}.  One can also
handcraft a better policy that chooses an action based on the past
history of actions and observations \cite{SilVen10}.  However, it is
often not easy to determine what the next action should be, given the
past history of actions and observations.  As in the case of upper
bounds, it is often more intuitive to work with states rather than
beliefs.  We describe a class of methods that we call
\emph{scenario-based policies}. In a scenario-based policy, we
construct a mapping $f\colon S \mapsto A$ that specifies an action at
a given state. We then specify a function that maps a belief to a
state $\Lambda\colon \mathcal{B} \mapsto S$ and let the default policy
be $\pi_0(b) = f(\Lambda(b))$.  As an example, let $\Lambda(b)$ be
the mode of the distribution $b$ (for a \despot node, this is the most
frequent start state under all scenarios in $\seqset_b$). We then let
$f$ be an optimal policy for the underlying MDP to obtain what we
call the \emph{mode-MDP} policy.

Scenario-based policies considerably ease the difficulty of
constructing effective default policies. However, depending on the
choice of $\Lambda$, they may not satisfy Theorem~\ref{th:uniform},
which assumes that the value of a default policy on one scenario is
independent of the value of the policy on another scenario. In
particular, the mode-MDP policy violates this assumption and may
overfit to the sampled scenarios. However, in practice, we expect the
benefit of being able to construct good default policies to usually
outweigh the concerns of overfitting with the default policy.

Given a default policy $\pi_0$, we obtain the initial lower bound
$\lowerv_0(b)$ at a \despot node $b$ by simulating $\pi_0$ for a finite number
of steps under each scenario $\seqset_b$ and calculating the average total
discounted reward.

\subsection{Analysis}\label{subsec:analysis}
The dynamic programming algorithm builds a full \despot \dtree.  The
anytime forward search algorithm builds a \despot incrementally and
terminates with a partial \despot $\dtree'$, which is a subtree of
\dtree.  The main objective of the analysis is to show that the
optimal regularized policy $\hat \pol$ derived from $\dtree'$
converges to the optimal regularized policy in \dtree.  Furthermore,
the performance of the anytime algorithm degrades gracefully even when
the upper bound $\upperv_0$ is not strictly admissible.

We start with some lemmas to justify the choices made in the anytime
algorithm.  Lemma~\ref{lem:forward} says that the excess uncertainty at a node
$b$ is bounded by the sum of excess uncertainty over its children under the
action branch $a^*$ that has the highest upper bound $\uppernu(b,a^*)$. This
provides a greedy means to reduce excess uncertainty by recursively exploring
the action branch $a^*$ and the observation branch with the highest excess
uncertainty, justifying \eqref{eq:8} and \eqref{eq:9} as the action and observation selection criteria.

\begin{lemma} \label{lem:forward}
For any \despot node $b$, if $\wdeu(b) > 0$ and $a^* =
  \arg\max_{a\in A} \uppernu(b, a)$, then
\[\wdeu(b) \le \sum_{z \in Z_{b, a^*}} \wdeu(b'),\]
where $b' = \tau(b,a^*,z)$ is a child of $b$. 
\end{lemma}
This lemma generalizes a similar result in \cite{SmiSim04} by taking
regularization into account.

Lemma~\ref{lem:blocking} justifies
$\textsc{Prune}$, which prunes the search when a node is
blocked by any of its ancestors. 
\begin{lemma} \label{lem:blocking}
Let $b'$ be an ancestor of $b$ in a \despot $\dtree$ and $\len(b',b)$ be the number of nodes on the path from $b'$ to
$b$. If
\[
\frac{|\seqset_{b'}|}{K} \gamma^{\depth(b')} (\upperv(b') - \lowerv_0(b'))
\le \lambda \cdot  \len(b', b),
\]
then $b$ cannot be a belief node in an optimal regularized policy derived from
$\dtree$.
\end{lemma}

We proceed to analyze the performance of the optimal regularized policy $\hat{\pol}$
derived from the partial \despot constructed.
The action output by the anytime \despot algorithm is the action
$\hat{\pol}(b_0)$, because the initialization and the computation of the lower
bound $\lowernu$ via the backup equations are exactly that for finding an
optimal regularized policy value in the partial \despot.

We now state the main results in the next two theorems. Both assume that
the initial upper bound $\upperv_0$ is \emph{$\delta$-approximate}:
\begin{dm}
  \upperv_0(b) \ge \hat{V}^*(b) - \delta,
\end{dm}
for every \despot node $b$.
If the initial upper bound is $0$-approximate, that is, it is indeed an upper
bound for $\hat{V}^*(b)$, then we say the heuristic is \emph{admissible}.
First, consider the case in which the maximum online planning time per
step \tmax is bounded. 

\begin{theorem}\label{th:bounded}
Suppose that \tmax is bounded and that the anytime \despot algorithm terminates
with a partial \despot $\dtree'$ that has gap $\gap(b_0)$ between the upper
and lower bounds at the root $b_0$. 
The optimal regularized policy $\hat{\pol}$ derived from $\dtree'$
satisfies
\begin{dm}
  \nu_{\hat{\pol}}(b_0) \ge \nu^*(b_0) - \gap(b_0) - \delta,
\end{dm}
where $\nu^*(b_0)$ is the value of an optimal regularized policy derived from the full
\despot \dtree at $b_0$. 
\end{theorem}
Since $\epsilon(b_0)$ decreases monotonically as \tmax grows, the above result
shows that the performance of $\hat \pol$ approaches that of an optimal
regularized policy as the running time increases. Furthermore,
the error in initial upper bound approximation affects the final result by at most $\delta$.
Next we consider unbounded maximum planning time \tmax.
The objective here is to show that despite unbounded \tmax,
the anytime algorithm terminates in finite time with a near-optimal or
\emph{optimal} regularized policy.  
\begin{theorem}\label{th:unbounded}
Suppose that \tmax is unbounded and $\gap_0$ is the target gap between the
upper and lower bound at the root of the partial \despot constructed by the
anytime \despot algorithm. Let $\nu^*(b_0)$ be
the value of an optimal regularized policy derived from the full
\despot \dtree at $b_0$.
\begin{itemize}
\item[(1)] If $\gap_0 > 0$, then the algorithm terminates in finite
  time with a near-optimal regularized policy $\hat \pol$ satisfying
\begin{dm}
\nu_{\hat{\pol}}(b_0) \ge \nu^*(b_0) - \gap_0 - \delta.
\end{dm}
\item[(2)] If $\gap_0 = 0$, $\delta=0$, and
  the regularization constant $\lambda>0$, then the 
  algorithm terminates in finite time with an optimal regularized policy $\hat \pol$, \ie, 
   $\nu_{\hat{\pol}}(b_0) = \nu^*(b_0)$.
\end{itemize}
\end{theorem}
In the case $\gap_0 > 0$, the algorithm aims for an approximately
optimal regularized policy. Compared with Theorem~\ref{th:bounded}, here the
maximum planning time limit \tmax is removed, and the algorithm
achieves the target gap $\gap_0$ exactly, after sufficient computation
time.  In the case $\gap_0 = 0$, the algorithm aims for an optimal
regularized policy. Two additional conditions are required to guarantee
finite-time termination and optimality. Clearly one is a true upper
bound with no approximation error, \ie, $\delta=0$. The other is a
strictly positive regularization constant. This assumption implies
that there is a finite optimal regularized policy, and thus allows the
algorithm to terminate in finite time.

\section{Experiments} 
\label{sec:experiments}
We now compare the anytime \despot algorithm with three
state-of-the-art POMDP algorithms (\secref{subsec:benchmark}). We also
study the effects of regularization (\secref{subsec:reg}) and initial
bounds (\secref{subsec:bounds}) on the performance of our algorithm.

\subsection{Performance Comparison}
\label{subsec:benchmark}
We compare \despot with SARSOP \cite{KurHsu08}, AEMS2
\cite{RosCha07,RosPin08}, and POMCP \cite{SilVen10}. SARSOP is one of
the fastest \emph{off-line} POMDP algorithms. While it cannot compete
with online algorithms on scalability, it often provides better
results on POMDPs of moderate size and helps to calibrate the
performance of online algorithms.  AEMS2 is an early successful online
POMDP algorithm. Again it is not designed to scale to very large state
and observation spaces and is used here as calibration on moderate
sized problems. POMCP scales up extremely well in practice
\cite{SilVen10} and allows us to calibrate the performance of \despot
on very large problems.

We implemented DESPOT and AEMS2 ourselves.
We used the authors' implementation of POMCP \cite{SilVen10},
but improved the implementation to support a very large number of observations and strictly
adhere to the time limit for online planning. We used the APPL 
package for SARSOP \cite{KurHsu08}.
All algorithms were implemented in C++.

For each algorithm, we tuned the key parameters on each domain through
offline training, using a data set distinct from the online test data
set, as we expect this to be the common usage mode for online planning.
Specifically, the regularization parameter $\lambda$ for
\despot was selected offline from the set $\{0, 0.01, 0.1, 1, 10\}$ by
running the algorithm with a training set distinct from the online
test set.  Similarly,  the exploration constant $c$ of POMCP
was  chosen from the set $\{1, 10, 100, 1000, 10000\}$ for the best
performance.  Other parameters of the algorithms are set to reasonable
values independent of the domain being considered.  Specifically, we
chose $\xi = 0.95$ as in SARSOP \cite{KurHsu08}.  We chose $D=90$ for
\despot because $\gamma^D \approx 0.01$ when $\gamma = 0.95$, which is
the typical discount factor used.  We chose $K = 500$, but a smaller
value may work as well.

All the algorithms were evaluated on the same experimental  platform.
The online POMDP
algorithms were given exactly 1 second per step to choose an action.

The test domains  range in size from small to extremely
large.  The results are reported in Table~\ref{table:results}.  In summary,
SARSOP and AEMS2 have good performance on the smaller domains, but cannot scale
up. POMCP scales up to very large domains, but has poor performance on some
domains. \despot has strong overall performance.  On the smaller domains, it
matches with SARSOP and AEMS2 in performance. On the large domains, it matches
and sometimes outperforms POMCP. The details on each domain are described
below.

\begin{table}
{\scriptsize
\caption{
Performance comparison. We  report the average total discounted reward
achieved. For Pocman, we follow  \cite{SilVen10} and report the average total reward without discounting.
A dash ``--'' indicates that an algorithm fails to run successfully on a
domain, because the state space or the observation space is too large, and
memory limit was exceeded.
For AEMS2 and POMCP, our experimental results sometimes differ from those
reported in earlier work, possibly due to
differences in experimental settings and platforms. 
We thus report both, with the results in earlier work
\cite{RosPin08,SilVen10}  in parentheses. 
}
\label{table:results}
\vspace*{-10pt}
\begin{center}
\begin{tabular}{l*{7}{@{\hspace{1.5ex}}c}}
\toprule
& \emph{Tag}
  & \emph{Laser Tag} & \emph{RS(7,8)}
  & \emph{RS(11,11)} & \emph{RS(15,15)}
  & \emph{Pocman} & \emph{Bridge Crossing} \\
\midrule
 $|S|$
   & 870     & ~~4,830                  &
   12,544         & 247,808            & 7,372,800         & $\sim 10^{56}$
   &  10 \\
 $|A|$    & ~~~~5    & ~~~~~~~~~5                      & ~~~~~~~13                 & ~~~~~~~~~16                 & ~~~~~~~~~~~~20                & ~~~~~~~4                              & ~~3     \\
 $|Z|$     & ~~30   & $\sim 1.5 \times 10^6$ & ~~~~~~~~~3     & ~~~~~~~~~~~3                  & ~~~~~~~~~~~~~~3                 & 1,024                 & ~~1         \\ \hline
SARSOP           & $-6.03\pm0.12$    & --                     & $21.47\pm
0.04$    & $21.56 \pm 0.11$   & --                & --  & $-7.40 \pm 0.0$                    \\ \hline
AEMS2        & $-6.41 \pm 0.28$  & --                     & $20.89 \pm 0.30$   & --                 & --                & --   & $-7.40 \pm 0.0$          \\
~                                        & ($-6.19\pm 0.15$) &                     & ($21.37 \pm 0.22$) &                  &                 &                   \\\hline
POMCP                                         & $-7.14\pm 0.28$   & $-19.58 \pm 0.06$      & $16.80 \pm 0.30$   & $18.10 \pm 0.36$   & $12.23 \pm 0.32$  & $294.16 \pm 4.06$  & $-20.00 \pm 0.0$           \\
~                                        &                 &                      & ($20.71\pm 0.21$)  & ($20.01 \pm 0.23$) & ($15.32\pm 0.28$) &   \\ \hline
DESPOT                                   & $-6.23 \pm 0.26$  & $-8.45 \pm 0.26$
                                         & $20.93  \pm 0.30$  & $21.75 \pm
                                         0.30$  & $18.64 \pm 0.28$       &
                                         $317.78 \pm 4.20$ & ~~$-7.40 \pm 0.0$  \\
\bottomrule
\end{tabular}
\end{center}
}
\end{table}

\begin{figure}
\begin{center}
  \begin{tabular}{c@{\hspace*{0.5in}}c}
\includegraphics[height=1.2in]{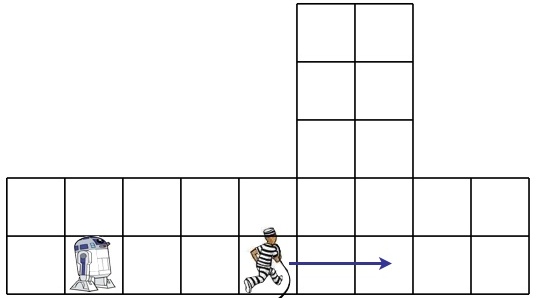} &
\includegraphics[height=1.4in]{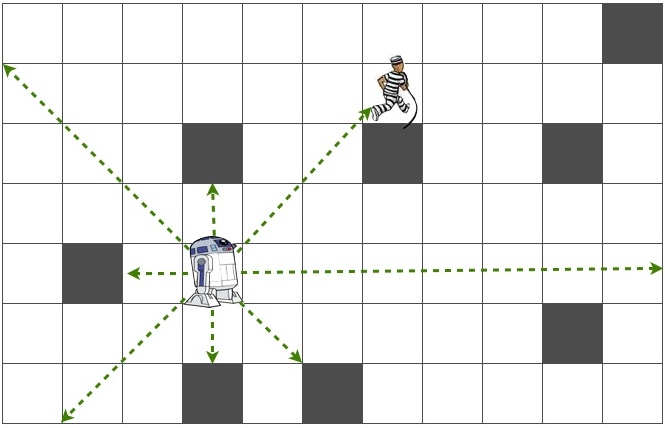} \\
(\subfig a) & (\subfig b) \\
\mbox{}\\
\includegraphics[height=1.5in]{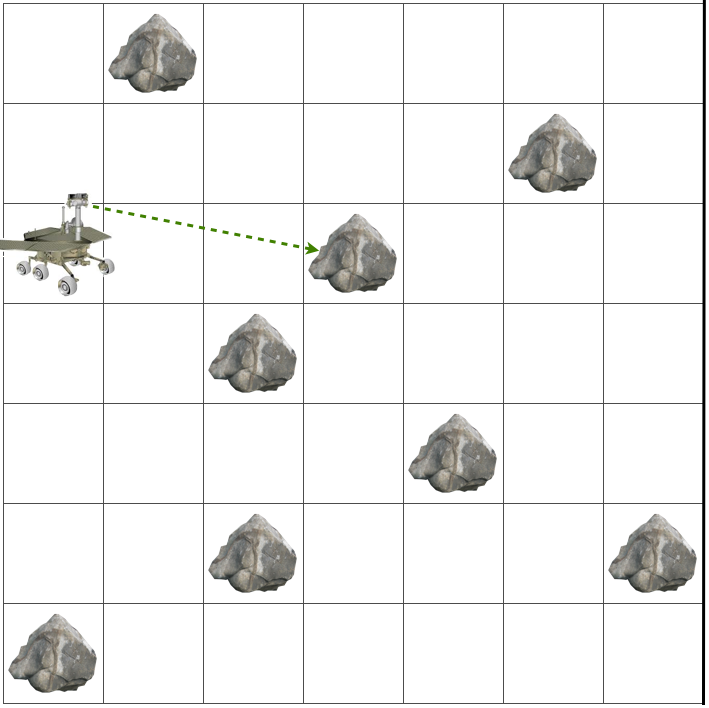} &
\includegraphics[height=1.6in]{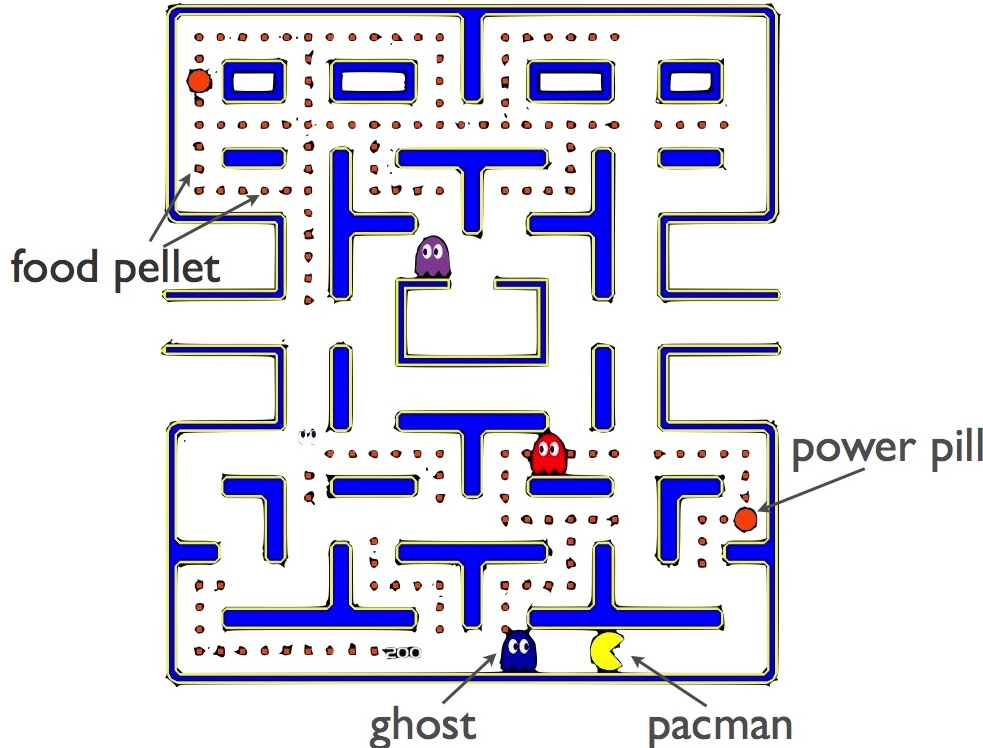} \\
(\subfig c) & (\subfig d) \\
  \end{tabular}
	\caption{Four test domains. (\subfig a) Tag. A robot chases
          an unobserved target that runs away. (\subfig b) Laser Tag.
          A robot chases
          a target in a $7 \times 11$ grid environment populated with
         obstacles. The robot is equipped 
          with a laser range finder for self-localization.  (\subfig c)
          Rock Sample. A robot rover senses rocks to identify ``good'' ones
          and samples them. Upon completion, it exits the east boundary.
            (\subfig d) The original Pacman game. }
	\label{fig:tests}
\end{center}
\end{figure}

\subsubsection{Tag}
  \label{sec:tag}
\emph{Tag} is a standard POMDP benchmark introduced by \citeA{pingor03}.
A robot and a
target operate in a grid with 29
possible positions (\figref{fig:tests}\subfig a).
The robot's goal is to find and tag the target that
intentionally runs away. They start at random initial positions. 
The robot knows its own position, but can observe the
target's position only if they are in the same grid cell. The robot can either
stay in the same position or move to the four adjacent positions, paying a
cost of $-1$ for each move. It can also attempt to tag the target.
It is rewarded
$+10$, if  the attempt is successful, and is penalized $-10$ otherwise.
To complete the task successfully, a good policy exploits the target's
dynamics to ``push'' it against a corner of the environment.

For \despot, we use the hindsight optimization bound for the initial upper
bound $\upperv_0$ and  initialize  hindsight optimization by setting
$U(D, s)$ to be the optimal MDP value (\secref{sec:upper}).
We use the mode-MDP policy for the default policy $\pol_0$
(\secref{sec:lower}).

POMCP cannot use the mode-MDP policy, as it requires default policies
that depend on the history only.  We use the Tag implementation that
comes as part of the authors' POMCP package, but improved its default
policy.  The original default policy tags when both the robot and the
target lie in a corner.  Otherwise the robot randomly chooses an
action that avoids doubling back or going into the walls.  The
improved policy tags whenever the agent and the target lie in the same
grid cell, otherwise avoids doubling back or going into the walls, yielding
better results based on our experiments.

On this moderate-size domain, SARSOP achieves the best result.
AEMS2  and \despot have comparable  performance. POMCP's performance is
much weaker, partly because of the limitation on its default policy.

\subsubsection{Laser Tag}
\label{sec:lasertag}

Theorem~\ref{th:uniform} suggests that \despot may perform well even when
the observation space is large, provided that
a small good policy exists.
We now consider \textit{Laser Tag}, an expanded version of Tag
with a large observation space. In Laser Tag, the robot
moves in a $7\times11$ rectangular grid with obstacles placed randomly in eight 
grid cells (\figref{fig:tests}\subfig b).
The  robot's and target's behaviors remain the same as before.
However, the robot does not know its own position exactly and is
distributed uniformly over the grid initially. To localize, it is
equipped with a laser range finder that measures the distances in
eight directions. The side length of each cell is 1. The laser
reading in each direction is generated from a normal distribution
centered at the true distance of the robot to the nearest obstacle in
that direction, with a standard deviation of $2.5$.  The readings are
rounded to the nearest integers.  So an observation comprises a set of
eight integers, and the total number of observations is roughly $1.5
\times 10^6$.

\despot uses a domain-specific method, which we
call \emph{Shortest Path} (SP) for the upper bound.  For every
possible initial target position, we compute an upper bound by
assuming that the target stays stationary and that the robot follows a
shortest path to tag the target.  We then take the average over the
sampled scenarios. \despot's default policy is similar to the one used
by POMCP in Tag, but it uses the most likely robot position to choose
actions that avoid doubling back and running into walls.  So it is 
not a scenario-based policy but a hybrid policy that makes use of both
the belief and the history.  

As the robot does not know its exact location, it is more difficult for POMCP's default policy to avoid going into walls and doubling back. Hence, we only implemented the action of tagging whenever the robot and target are in the same location, and did not implement wall and doubling back avoidance.

With the very large observation space, we are not able to successfully run SARSOP and AEMS2.  \despot achieves substantially better result than 
POMCP on this task.

\subsubsection{Rock Sample}

Next we consider \textit{Rock Sample}, a well-established benchmark
with a large state space \cite{SmiSim04}.  In {RS}$(n, k)$, a robot
rover moves on an $n\times n$ grid containing $k$ rocks, each of which
may be \emph{good} or \emph{bad} (\figref{fig:tests}\subfig c).  The
robot's goal is to visit and sample the good rocks, and exit the east
boundary upon completion. At each step, the robot may move to an
adjacent cell, sense a rock, or sample a rock. Sampling gives a reward
of $+10$ if the rock is good and $-10$ otherwise. Moving and sensing
have reward $0$.  Moving or sampling do not produce any observation,
or equivalently, null observation is produced. Sensing a rock produces an
observation, {\small GOOD} or {\small BAD}, with probability of being
correct decreasing exponentially with the robot's distance from the
rock.
To obtain high total reward, the robot navigates the environment and senses
rocks to identify the good ones; at the same time, it exploits the
information gained to visit and sample the good rocks.

For upper bounds, \despot uses the MDP upper bound, which is a true upper
bound in this case. For default policy, it uses a simple fixed-action default policy that always moves to the east.

POMCP uses the default policy described in \cite{SilVen10}.  The robot travels
from rock location to rock location.  At each rock location, the robot samples
the rock if there are more {\small GOOD} observations there than {\small BAD}
observations.
If all remaining rocks have a greater number of {\small BAD} observations,
the robot moves to the east boundary and exits.

On the smallest domain {RS}$(7,8)$, the offline algorithm SARSOP obtains the
best result overall. Among the three online algorithms, \despot has the best
result. On {RS}$(11,11)$, \despot matches SARSOP in performance and is better
than POMCP. On the largest instance {RS}$(15,15)$, SARSOP and AEMS2 cannot
be completed successfully. 
It is also interesting to note that although the default policy for \despot is
weaker than that of POMCP, \despot still achieves better results.

\subsubsection{Pocman}\label{sec:pocman}

\emph{Pocman} \cite{SilVen10} is a partially observable variant of the
popular video game \emph{Pacman} (\figref{fig:tests}\subfig d).
In Pocman, an agent and four ghosts move in
a $17\times 19$ maze populated with food pellets.  Each agent move incurs a
cost of $-1$. Each food pellet provides a reward of $+10$.  If the agent is
captured by a ghost, the game terminates with a penalty of $-100$. In
addition, there are four power pills.
Within the next 15 time steps
after eating a power pill, the agent can eat a ghost encountered and
receives a reward of $+25$.  A ghost chases the agent if the agent is within a
Manhattan distance of $5$, but runs away if the agent possesses a power pill.
The agent does not know the exact ghost locations, but receives information on
whether it sees a ghost in each of the cardinal directions, on whether it
hears a ghost within a Manhattan distance of $2$, on whether it feels a wall
in each of the four cardinal directions, and on whether it smells food pellets
in adjacent or diagonally adjacent cells.  Pocman has an extremely large
state space of roughly $10^{56}$ states.

For \despot, we compute an approximate hindsight optimization bound at
a \despot node $b$ by summing the following quantities for each
scenario in $\seqset_b$ and taking the average over all scenarios: the
reward for eating each pellet discounted by its distance from pocman,
the reward for clearing the level discounted by the maximum distance
to a pellet, the default per-step reward of $-1$ for a number of steps
equal to the maximum distance to a pellet, the penalty for eating a
ghost discounted by the distance to the closest ghost being chased if
any, the penalty for dying discounted by the average distance to the
ghosts, and half the penalty for hitting a wall if the agent tries to
double back along its direction of movement. Under the default policy,
when the agent detects a ghost, it chases a ghost if it possess a
power pill and runs away from the ghost otherwise.  When the agent
does not detect any ghost, it makes a random move that avoids doubling
back or going into walls.  Due to the large scale of this domain, we
used $K=100$ scenarios in the experiments in order to stay within the
allocated $1$ second online planning time. For this problem, POMCP
uses the same default policy as \despot.

On this large-scale domain, \despot has slightly better performance than POMCP,
while SARSOP and AEMS2 cannot run successfully. 

\subsubsection{Bridge Crossing}
\label{subsec:bridge}

The experiments above indicate that both POMCP and \despot can handle very
large POMDPs. However, the UCT search strategy \cite{KocSze06} which POMCP
depends on has very poor worst-case behavior \cite{CoqMun07}.  We designed
\emph{Bridge Crossing}, a very simple domain, to illustrate this. 

In Bridge Crossing, a person attempts to cross a narrow bridge over a
mountain pass in the dark. He starts out at one end of bridge,
but is uncertain of his exact initial position because of the
darkness. At any time, he may call for rescue and terminate the attempt.
In the POMDP model, the man has
$10$ discretized positions $x \in \{0, 1, \ldots, 9\}$ along the
bridge, with the person at the end of $x=0$. He is uncertain about his initial position, with a maximum
error of $1$.  He can move forward or backward, with cost
$-1$. However, moving forward at $x=9$ has cost $0$, indicating
successful crossing.  For simplicity, we assume no movement noise.
The person can call for rescue, with cost $-x-20$, where $x$ is his
current position.  The person has no observations while in the middle of
the bridge.  His attempt terminates when he successfully crosses the
bridge or calls for rescue.

\despot uses the uninformed upper bound and the trivial default policy
of calling for rescue immediately. POMCP uses the same default policy. 

This is an open-loop planning problem. A policy is simply a sequence
of actions, as there are no observations.  There are $3$ possible
actions at each step, and thus when considering policies of length at most 10,
there are at most \emph{$3^{1} + 3^{2} + \ldots + 3^{10} < 3^{11}$} policies.
A simple breadth-first enumeration of these policies is sufficient to identify
the optimal one: keep moving forward for a maximum of $10$ steps.
Indeed, both SARSOP and AEMS2 obtain this optimal policy, as does
\despot. While the initial upper bound and the default policy for
\despot are uninformed, the backup operations improve the bounds and
guide the search towards the right direction to close the gap between
the upper and lower bounds.  In contrast, POMCP's performance on this
domain is poor, because it is misled by its default policy and has an
extremely poor convergence rate for cases such as this. While the optimal
policy is to move forward all the way, the Monte
Carlo simulations employed by POMCP suggests doing exactly the
opposite at each step---move backward and call for rescue---because calling for
rescue early near the starting point  incurs lower
cost. Increasing the  exploration constant $c$ may somewhat
alleviate this difficulty, but does not resolve it substantively.

\subsection{Benefits of Regularization}
\label{subsec:reg}
\begin{table}
\caption{Performance of \despot, with and without regularization.
  The table reports the average total discounted reward achieved. For Pocman,
  it reports the total undiscounted reward without discounting.
	} 
\label{tbl:reg}
\vspace*{-15pt}
\begin{center}
{\scriptsize
\begin{tabular}{l*{8}{@{\hspace{1.5ex}}c}}
  \toprule
  & \emph{Tag} & \emph{Laser Tag} & \emph{RS(7,8)} & \emph{RS(11,11)} &
  \emph{RS(15,15)} & \emph{Pocman} & \emph{Bridge Crossing}\\ \midrule
 $|Z|$     & 30   & $\sim 1.5 \times 10^6$ & ~~~3     & ~~3  & ~~~3                 & 1,024                 & ~~1         \\ \midrule
$\lambda$ & 0.01 & 0.01 & 0.0 & 0.0 & 0.0 & 0.1 & 0.0 \\
  Regularized  & $-6.23 \pm 0.26$ & $-8.45 \pm 0.26$ & $20.93
  \pm 0.30$ & $21.75 \pm 0.30$ & $18.64 \pm 0.28$ & $317.78 \pm 4.20$
  & $-7.40 \pm 0.0$ \\ \midrule
  Unregularized & $-6.48 \pm 0.26$ & $-9.95 \pm 0.26$ & $20.90
  \pm 0.30$ & $21.75 \pm 0.30$ & $18.15 \pm 0.29$ & $269.64 \pm 4.33$
  & $-7.40 \pm 0.0$ \\

  \bottomrule
\end{tabular}
}
\end{center}
\end{table}

We now study the effect of regularization on the performance of
\despot. 
If a POMDP has a large number of observations, the size of the corresponding
belief tree and the optimal policy may be large as well. 
Overfitting to the sampled scenarios is thus more likely to occur, and we would
expect that regularization may help.
Indeed, Table~\ref{tbl:reg}
shows that
 Tag, Rock
Sample, and Bridge Crossing, which all have a small or moderate
number of observations, do not benefit much from 
regularization.  The remaining two, Laser Tag and Pocman, which have
a large number of observations, benefit more significantly from
regularization.

To understand better the underlying cause, we designed another simple
domain, \emph{Adventurer}, with a variable number of observations.
An adventurer explores an ancient ruin modeled as a $1 \times 5$ grid.
He is initially located in the leftmost cell. The treasure is located in the
rightmost cell, with value uniformly distributed over a finite set
$X$.  If the adventurer reaches the rightmost cell and stays there for one
step to dig up the treasure, he receives the value of the treasure as the
reward.  The adventurer can drive left, drive right, or stay in the same
place.  Each move may severely damage his vehicle with probability $0.5$,
because of the rough terrain.  The damage incurs a cost of $-10$ and
terminates the adventure. The adventurer has a noisy sensor that reports the
value of the treasure at each time step.  The sensor reading is accurate with
probability $0.7$, and the noise is evenly distributed over other possible
values in $X$.  At each time step, the adventurer must decide, based on the
sensor readings received so far, whether to drive on in hope of getting the
treasure or to stay put and protect his vehicle.
With a discount factor of 0.95, the optimal policy is in fact to stay in the
same place.

We studied two settings empirically: $X= \{101, 150\}$ and
$X = \{101, 102, \ldots, 150\}$, which result in $2$ and $50$ observations,
respectively. For each setting, we constructed 1,000 full DESPOTs from $K=500$
randomly sampled scenarios. 
We compute the optimal action without regularization for each DESPOT.
In the first setting with 2 observations, the
computed action is always to stay in the same place. This is indeed optimal.
In the second setting with 50 observations, about half of the computed actions
are to move right. This is suboptimal. Why did this happen?

Recall that the sampled scenarios are split among the observation
branches. With $2$ observations, each observation branch has $250$ scenarios
on average, after the first step. This is sufficient to represent the
uncertainty well. With $50$ observations, each observation branch has only
about $5$ scenarios on average. Because of the sampling variance, the
algorithm easily underestimates the probability and the expected cost of
damage to the vehicle. In other words, it overfits to the sampled scenarios.
Overfitting thus occurs much more often with a large number of observations.

Regularization helps to alleviate overfitting. For the second setting, we ran
the algorithm with and without regularization.  Without regularization,
the average total discounted reward is $-6.06 \pm 0.24$.
With regularization, the average total discounted reward is 
$0\pm 0$, which is optimal.

\subsection{Effect of Initial Bounds} 
\label{subsec:bounds}

Both \despot and POMCP rely on Monte Carlo simulation as a key algorithmic
technique. They differ in two main aspects. One is their search strategies. We
have seen in \secref{subsec:benchmark} that POMCP's search strategy is more
greedy, while \despot's search strategy is more robust. Another difference is
their default policies. While POMCP requires history-based default policies,
\despot provides greater flexibility. Here we look into the benefits of this
flexibility and also illustrate several additional techniques for constructing
effective default policies and initial upper bounds.  The benefits of
cleverly-crafted default policies and initial upper bounds are
domain-dependent.  We examine three domains --- Tag, Laser Tag, and Pocman --- to
give some intuition.  Tag and Laser Tag are closely related for easy
comparison, and both Laser Tag and Pocman are large-scale domains.
The results for Tag, Laser Tag, and Pocman are shown on Tables~\ref{tab:tag}
and \ref{tab:pocman}.

For  Tag and Laser Tag, 
we considered the following default policies for both problems.
\begin{itemize}
\item NORTH is a domain-specific,  fixed-action policy that always moves the robot to the
  adjacent grid cell to the north at every time step.
\item HIST is the improved history-based policy used by POMCP
  in the Tag experiment (\secref{sec:tag}). 
\item HYBRID is the hybrid policy used by DESPOT in the Laser Tag experiment
  (\secref{sec:lasertag}). 
\item The mode-MDP policy is the scenario-based policy described in
  \secref{sec:lower} and used by \despot in the Tag experiment (\secref{sec:tag}).
  It first estimates the most likely state and then applies the optimal MDP
  policy accordingly.
\item MODE-SP is a variant of the MODE-MDP policy. Instead of
  the optimal MDP policy, it applies a handcrafted domain-specific
  policy, which takes the action that moves the robot towards the target along
  a shortest path in the most likely state.
\end{itemize}
Next, consider the initial upper bounds.  UI is the uninformed upper
bound \eqref{eq:uninformed}.  MDP is the MDP upper bound \eqref{eq:mdpbound}.
SP is the domain-specific bound used in Laser Tag (\secref{sec:lasertag}).
We can use any of these three bounds, UI, MDP, or SP, to
initialize $U(D,s)$ and obtain a corresponding hindsight optimization bound.

\begin{table}
  \caption{
    Comparison of different initial upper and default policies on Tag and
    Laser Tag. 
    The table reports the average total discounted rewards of the default
    policies (column 2) and  of DESPOT
    when used with
    different combinations of initial upper and default policies (columns 3--8).
  } \label{tab:tag}
\begin{center}
{\scriptsize
\begin{tabular}{lrc*{6}{@{\hspace{1.5ex}}r}}
\toprule
\multicolumn{2}{c}{Default Policy} &&
\multicolumn{6}{c}{Initial Upper Bound}\\  \cline{4-9}
&   &&
 \multicolumn{1}{c}{UI}  &
 \multicolumn{1}{c}{MDP}  &
 \multicolumn{1}{c}{SP}  &
 \multicolumn{1}{c}{HO-UI}  &
 \multicolumn{1}{c}{HO-MDP}  &
 \multicolumn{1}{c}{HO-SP}  \\
\midrule
\emph{\textbf{Tag}}\\
 ~~NORTH & -19.80 $\pm$ 0.00   && -15.29 $\pm$ 0.35 & -6.27 $\pm$ 0.26 & -6.31
$\pm$ 0.26 & -6.39 $\pm$ 0.27 &  -6.49 $\pm$ 0.26 & -6.45 $\pm$ 0.26 \\
~~HIST & -14.95 $\pm$ 0.41 && -8.29 $\pm$ 0.28  & -7.36 $\pm$ 0.26 & -7.25 $\pm$ 0.26 & -7.34 $\pm$ 0.26 & -7.41 $\pm$ 0.26 & -7.41 $\pm$ 0.26 \\
~~MODE-MDP & -9.31 $\pm$ 0.29  && -6.29 $\pm$ 0.27  & -6.27 $\pm$ 0.26 & -6.24 $\pm$ 0.27 & -6.28 $\pm$ 0.26 & -6.23 $\pm$ 0.26 & -6.32 $\pm$ 0.26  \\
~~MODE-SP & -12.57 $\pm$ 0.34  && -6.53 $\pm$ 0.27  & -6.51 $\pm$ 0.27 & -6.38 $\pm$ 0.26 & -6.52 $\pm$ 0.27 & -6.56 $\pm$ 0.27 & -6.55 $\pm$ 0.26 \\
  \midrule
\emph{\textbf{Laser Tag}}\\
~~NORTH & -19.80 $\pm$ 0.00  &  & -15.80 $\pm$ 0.35 & -11.18 $\pm$ 0.28& -10.35
$\pm$ 0.26 & -10.91 $\pm$ 0.27 & -10.91 $\pm$ 0.27  & -10.91 $\pm$ 0.27 \\
~~HYBRID & -19.77 $\pm$ 0.27    &   & -8.76 $\pm$ 0.25  & -8.59 $\pm$ 0.26 & -8.45 $\pm$ 0.26  & -8.52 $\pm$ 0.25  & -8.66 $\pm$ 0.26   & -8.54 $\pm$ 0.26  \\
~~MODE-MDP & -9.97 $\pm$ 0.25 &   & -9.83 $\pm$ 0.25  & -9.84 $\pm$ 0.25 & -9.83 $\pm$ 0.25  & -9.83 $\pm$ 0.25  & -9.83 $\pm$ 0.25   & -9.83 $\pm$ 0.25   \\
~~MODE-SP & -10.07 $\pm$ 0.26  & & -9.63 $\pm$ 0.25  & -9.63 $\pm$ 0.25 & -9.63 $\pm$ 0.25  & -9.63 $\pm$ 0.25  & -9.63 $\pm$ 0.25   & -9.63 $\pm$ 0.25  \\
  
 \bottomrule
\end{tabular}
}
\end{center}
\end{table}

For Pocman, we considered three default policies.  NORTH is the policy of always
moving to cell to the north.  RANDOM is the policy which moves randomly to a
legal adjacent cell. REACTIVE is the policy
described in Section~\ref{sec:pocman}. 
We also considered two initial upper bounds.
UI is the uninformed upper bound \eqref{eq:uninformed}. AHO is the
approximate hindsight optimization bound described in Section~\ref{sec:pocman}.

\begin{table}
  \caption{
    Comparison of different initial upper and default policies on Pocman.
    The table reports the average total undiscounted rewards of the default
    policies (column 2) and  of DESPOT
    when used with
    different combinations of initial upper and default policies (columns 3--4).
}\label{tab:pocman} 
\begin{center}
{\scriptsize

\begin{tabular}{lrrrr}
\toprule
\multicolumn{2}{c}{Default Policy} &&
\multicolumn{2}{c}{Initial Upper Bound}\\  \cline{4-5}
&  &&
  \multicolumn{1}{c}{UI} &
  \multicolumn{1}{c}{AHO} \\
\midrule
  NORTH    & $-1253.91 \pm 24.03$ &  & $-82.09 \pm 1.72$ & $-8.04 \pm 3.54$ \\
  RANDOM   & $-72.68 \pm 1.42$    &  & $284.70 \pm 4.45$ & $202.21 \pm 4.71$ \\
  REACTIVE & $80.93 \pm 5.40$     &  & $315.62 \pm 4.16$ & $317.78 \pm 4.20$  \\
\bottomrule
\end{tabular}
}
\end{center}
\end{table}

The results in Tables~\ref{tab:tag} and \ref{tab:pocman} offer several
observations.  First, \despot's online search considerably improves the
default policy that it starts with, regardless of the specific default policy
and upper bound used.  This indicates the importance of online search.
Second, while some initial upper bounds are approximate, they nevertheless
yield very good results, even compared with true upper bounds.  This is
illustrated by the approximate MDP bounds for Tag and Laser Tag as well as the
approximate hindsight optimization bound for Pocman.  Third, the simple
uninformed upper bound can sometimes yield good results, when paired with
suitable default policies. Finally, these observations are consistent across
domains of different sizes, \eg, Tag and Laser Tag.

In these examples, default policies seem to have much more impact than initial
upper bounds do.
It is also interesting to note that for LaserTag, HYBRID is one of the
worst-performing default policies, but leads to the best performance
ultimately.  So a stronger default policy does not always lead to better
performance. There are other factors that may affect the performance, and
flexibility in constructing both the upper bounds and initial policies can be
useful.

\section{Discussion}
\label{sec:discussion}

One key idea of \despot is to use a set of randomly sampled
scenarios as an approximate representation of uncertainty and plan over the
sampled scenarios. This works well if there exists a compact, near-optimal
policy. The basic idea extends beyond POMDP planning and applies to other
planning tasks under uncertainty, \eg, MDP planning and belief-space MDP planning.

The search strategy of the anytime \despot algorithm has several
desirable properties. It is asymptotically sound and complete (Theorem~\ref{th:unbounded}).
It is robust against imperfect heuristics (Theorem~\ref{th:unbounded}).  It is
also flexible and allows easy incorporation of domain knowledge.  However,
there are many alternatives, \eg, the UCT strategy used in POMCP. While UCT
has very poor performance in the worst case, it is simpler to implement, an
important practical consideration. Further, it avoids the overhead of
computing the upper and lower bounds during the search and thus could be
 more efficient in some cases. 
In practice, there is trade-off between simplicity and robustness. 
 
 Large observation spaces may cause  \despot to overfit to the sampled
 scenarios (see \secref{subsec:reg}) and pose a major challenge.
 Unfortunately, this may happen with many common sensors, such as cameras and
 laser range finders.  Regularization alleviates the difficulty. We are
 currently investigating several other approaches: structure the observation
 space hierarchically and importance sampling.

 The default policy plays an important role in  \despot.
A good default policy reduces the size of the optimal policy.
It also helps guide the heuristic search in the anytime search algorithm.
In practice, we want the default policy to incorporate as much domain
knowledge as possible for good performance.  Another interesting direction is
to learn the default policy during the search.

\section{Conclusions}
\label{sec:conclusions}

This paper presents \despot, a new approach to online POMDP planning.  The
main underlying
idea is to plan according to a set of sampled scenarios while avoiding
overfitting to the samples. Theoretical analysis shows that a \despot
compactly captures the ``execution'' of all policies under the sampled
scenarios and yields a near-optimal policy, provided that there is a small
near-optimal policy.  The analysis provides the justification for
our overall planning approach based on sampled scenarios and the need for
regularization.  Experimental results indicate strong performance of the
anytime \despot algorithm in practice.  On moderately-sized POMDPs, \despot is
competitive with SARSOP and AEMS2, but it scales up much better. On
large-scale POMDPs with up to $10^{56}$ states, \despot matches and sometimes
outperforms POMCP. 

\newpage
\bigskip\noindent{\bf Acknowledgments}

\medskip\noindent
The work was mainly performed while N. Ye was with the
Department of Computer Science, National University of Singapore.  We are
grateful to the anonymous reviewers for carefully reading the manuscript and
providing many suggestions which helped greatly in improving the paper.  The
work is supported in part by National Research Foundation Singapore through
the SMART IRG program, US Air Force Research Laboratory under agreements
FA2386-12-1-4031 and FA2386-15-1-4010, a Vice Chancellor's Research Fellowship
provided by Queensland University of Technology and an Australian Laureate
Fellowship (FL110100281) provided by Australian Research Council.

\appendix
\section{Proofs}
\subsection*{Proof of Theorem~\ref{th:uniform}}
We will need the following lemma from \cite[p.~103]{Hau92} Lemma 9, part (2).

\begin{lemma}
\emph{(Haussler's bound)}
\label{lem:haussler}
Let $Z_1,\ldots,Z_n$ be i.i.d random variables with range $0\leq Z_i\leq M$,
$\mathbb{E}(Z_i) = \mu$, and $\hat{\mu} = \frac{1}{n}\operatorname*{\sum}_{i=1}^{n}Z_i$, $1 \leq i \leq n$. Assume $\nu > 0$ and $0 < \alpha < 1$. Then
\begin{equation}
	\Pr\left(d_\nu(\hat{\mu}, \mu) > \alpha\right) < 2e^{-\alpha^2{\nu}n/M} \nonumber
\end{equation}
\noindent where $d_\nu(r, s) = \frac{|r - s|}{\nu + r + s}$.
As a consequence,
\begin{equation}
	\Pr\left(\mu < \frac{1-\alpha}{1+\alpha}\hat{\mu} -
	  \frac{\alpha}{1+\alpha}\nu \right) < 2e^{-\alpha^2{\nu}n/M}.\nonumber
\end{equation}
\end{lemma}

Let $\Pi_i$ be the class of policy trees in $\Pi_{b_0, \height, \nsam}$ and having
size $i$. The next lemma bounds the size of $\Pi_i$.
\begin{lemma} \label{lem:size}
$|\Pi_i| \leq i^{i-2} (|A||Z|)^{i}$. 
\end{lemma}
\begin{proof}
Let $\Pi'_i$ be the class of rooted ordered trees of size $i$.
A policy tree in $\Pi_i$ is obtained from some tree in $\Pi'_i$ by assigning one
of the $|A|$ possible action labels to each node, and one of at most $|Z|$
possible labels to each edge, thus $|\Pi_i| \le |A|^{i} \cdot |Z|^{i-1}$.
 $|\Pi'_i|$.
To bound $|\Pi'_i|$, note that it is not more than the number of all trees
with $i$ labeled nodes, because the in-order labeling of a tree in $\Pi'_i$
corresponds to a labeled tree.
By Cayley's formula~\citeyear{cay89}, the number of trees with $i$
labeled nodes is $i^{i-2}$, thus $|\Pi'_i| \le i^{(i-2)}$.
Therefore
$|\Pi_i| 
 \le i^{i-2} \cdot |A|^i \cdot  |Z|^{i-1} 
 \le i^{i-2}(|A||Z|)^{i}$.
\end{proof}

In the following, we often abbreviate $V_{\pi}(b_0)$ and $\hat{V}_{\pi}(b_0)$
as $V_{\pi}$ and $\hat{V}_{\pi}$ respectively, since we will only consider the
true and empirical values for a fixed but arbitrary $b_0$.
Our proof follows a line of reasoning similar to that of \citeA{wanwon12}.

\begin{theoremstar}~{\hspace{-1ex}\rm\bf \ref{th:uniform}}
  For any $\tau, \alpha \in (0,1)$, any 
  belief $b_0$, and any positive integers $D$ and $K$, 
	with probability at least $1 - \tau$, \emph{every} \despot
  policy tree $\pol
  \in \polclass_{b_0, \height,\nsam}$ satisfies
\begin{equation*}
V_{\pi}(b_0) \geq
  \frac{1-\alpha}{1+\alpha}\hat{V}_{\pi}(b_0)
	- \frac{\rmax}{(1+\alpha)(1-\gamma)}\cdot\frac{\ln(4 / \tau) +
       |\pi|\ln \bigl(KD |A| |Z|\bigr)}{\alpha K},
\end{equation*}
where the random variable $\hat{V}_{\pi}(b_0)$ is the estimated value of $\pi$
under the set of $\nsam$ scenarios randomly sampled according to $b_0$.
\end{theoremstar}

\begin{proof}
  Consider an arbitrary policy tree $\pi \in \polclass_{b_0, \height,\nsam}$.
We know that for a random scenario $\seq$ for the belief $b_0$, executing the
policy $\pi$ w.r.t. $\seq$ gives us a sequence of states and observations
distributed according to the distributions $P(s' | s, a)$ and $P(z | s, a)$. 
Therefore, for $\pi$, its true value $V_{\pi}$ equals 
$\mathbb{E}\left(V_{\pi, \seq} \right)$, where the expectation is over the
distribution of scenarios.
On the other hand, since 
$\hat{V}_{\pi} = \frac{1}{K}\sum_{k=1}^K V_{\pi,\seq_k}$,
and the scenarios $\seq_1, \seq_2,\ldots,\seq_K$ are independently sampled,
Lemma~\ref{lem:haussler} gives
\begin{equation} \label{eqn:single}
	\Pr\left(V_{\pi} < \frac{1-\alpha}{1+\alpha} \hat{V}_{\pi} 
		- \frac{\alpha}{1 + \alpha} \epsilon_{|\pi|}\right) 
	< 2e^{-\alpha^2{\epsilon_{|\pi|}}K/M} 
\end{equation}
where $M = (R_{\max}) / (1-\gamma)$, and $\epsilon_{|\pi|}$ is chosen such that

\begin{equation} \label{eqn:epsilon}
2e^{-\alpha^2{\epsilon_{|\pi|}}K/M} = \tau/ (2|\pi|^2 |\Pi_{|\pi|}|).
\end{equation}

By the union bound, we have
\begin{eqnarray*} \label{eqn:union} 
\Pr\left( \exists \pi \in
    \polclass_{b_0, \height,\nsam} \text{ such that } V_{\pi} <
      \frac{1-\alpha}{1+\alpha} \hat{V}_{\pi} 
			- \frac{\alpha}{1 + \alpha} \epsilon_{|\pi|} 
  \right) \\
  \leq \sum_{i=1}^{\infty} \sum_{\pi \in \Pi_i} \Pr\left( V_{\pi} <
    \frac{1-\alpha}{1+\alpha} \hat{V}_{\pi} 
		- \frac{\alpha}{1 + \alpha} \epsilon_{|\pi|} \right).
\end{eqnarray*}

By the choice of $\epsilon_i$'s and Inequality~(\ref{eqn:single}), the right
hand side of the above inequality is bounded by
$\sum_{i=1}^{\infty} |\Pi_i| \cdot [ \tau / (2i^2 |\Pi_i|) ] 
= \bfgreek{pi}^2 \tau / 12 < \tau$, where the well-known identity
$\sum_{i=1}^\infty 1/i^2 = \bfgreek{pi}^2 / 6$ is used.
Hence,

\begin{equation} \label{eqn:uniform}
\Pr\left(
 \exists \pi \in \polclass_{b_0, \height,\nsam}
 \left[
				V_{\pi} < \frac{1-\alpha}{1+\alpha} \hat{V}_{\pi} 
					- \frac{\alpha}{1 + \alpha} \epsilon_{|\pi|}
			  \right]
		\right) 
< \tau.
\end{equation}
 
Equivalently, with probability $1-\tau$, every $\pi \in \polclass_{b_0, \height,\nsam}$
satisfies

\begin{equation} \label{eqn.bound}
V_{\pi} \geq \frac{1-\alpha}{1+\alpha}\hat{V}_{\pi} 
		- \frac{\alpha}{1+\alpha} \epsilon_{|\pi|}.
\end{equation}

To complete the proof, we now give an upper bound on $\epsilon_{|\pi|}$.
From Equation~(\ref{eqn:epsilon}), we can solve for $\epsilon_{|\pi|}$ to get
$\epsilon_{i} = \frac{R_{\max}}{\alpha (1-\gamma)} \cdot 
  \frac{\ln(4 / \tau) + \ln(i^2 |\Pi_i|)}{\alpha K}$.
For any $\pi$ in $\polclass_{b_0, \height,\nsam}$, its size is at most $KD$, and 
$i^2 |\Pi_i| \le (i |A| |Z|)^i \le (KD|A||Z|)^i$ by Lemma~\ref{lem:size}. Thus we have
\begin{equation}
\epsilon_{|\pi|} \le \frac{R_{\max}}{\alpha (1-\gamma)} \cdot 
  \frac{\ln(4 / \tau) + |\pi|\ln(KD |A| |Z|)}{\alpha K}. \nonumber
\end{equation}
Combining this with Inequality~(\ref{eqn.bound}), we get
\begin{equation}
V_{\pi} \geq \frac{1-\alpha}{1+\alpha}\hat{V}_{\pi} 
	- \frac{R_{\max}}{(1+\alpha)(1-\gamma)}
  \cdot \frac{\ln(4 / \tau) + |\pi|\ln(KD|A||Z|)}{\alpha K}. \nonumber
\end{equation}
This completes the proof.
\end{proof}

\subsection*{Proof of Theorem~\ref{th:despot}}
We need the following lemma for proving Theorem~\ref{th:despot}.

\begin{lemma}
\label{lem:hoeffding}
For a fixed policy $\pi$ and any $\tau \in (0,1)$, with probability at least $1 - \tau$.
\begin{equation}
	\hat{V}_{\pi} \geq V_\pi - \frac{R_{\max}}{1-\gamma}\sqrt{\frac{2\ln (1/\tau)}{K}}\nonumber
\end{equation}
\end{lemma}

\begin{proof}
Let $\pi$ be a policy and $V_\pi$ and $\hat{V}_\pi$ as mentioned. Hoeffding's inequality~\citeyear{hoe63} gives us
\begin{equation}
	\Pr\left(\hat{V}_{\pi} \ge V_{\pi} - \epsilon\right) \ge 1- e^{-K\epsilon^2 / (2M^2)}, \nonumber
\end{equation}
where $M = R_{\max} / (1 - \gamma)$.

Let $\tau = e^{-K\epsilon^2 / (2M^2)}$ and solve for $\epsilon$, then we get
\begin{equation}
	\Pr\left(\hat{V}_{\pi} \geq V_\pi - \frac{R_{\max}}{1-\gamma}\sqrt{\frac{2\ln (1/\tau)}{K}}\right) 
	 \geq 1-\tau. \nonumber
\end{equation}
\end{proof}

\begin{theoremstar}~{\hspace{-1ex}\rm\bf \ref{th:despot}}
  Let $\pol$ be an arbitrary policy at a belief $b_0$.  Let
  $\Pi_\mathcal{D}$ be the set of policies derived from a
\despot~$\mathcal{D}$ that has height \height and is constructed from
$\nsam$ scenarios sampled randomly according to $b_0$.  For any $\tau,
\alpha \in (0,1)$, if
\begin{equation*}
    \small
{\hat \pol} = \argmax_{\pol' \in \Pi_\mathcal{D}} \biggl\{   \frac{1-\alpha}{1+\alpha}\hat V_{\pi'}(b_0) -
     \frac{R_{\max}}{(1+\alpha)(1-\gamma)}\cdot\frac{|\pi'|\ln\bigl(KD|A||Z|\bigr)}{\alpha
        K}\biggr\},
  \end{equation*}
then
\begin{dm}
  \resizebox{0.95\hsize}{!}{$
  V_{\hat \pi}(b_0) \geq \frac{1-\alpha}{1+\alpha}V_\pi(b_0) -
  \frac{R_{\max}}{(1+\alpha)(1-\gamma)}\biggl(\frac{\ln(8 / \tau) +
      |\pi|\ln\bigl(KD|A||Z|\bigr)}{\alpha K} + (1-\alpha)\Bigl(\sqrt{\frac{2\ln
          (2/\tau)}{K}}+\gamma^{\sss D}\Bigr)\biggr),
  $}
\end{dm}
with probability at least $1 - \tau$.
\end{theoremstar}

\begin{proof}
By Theorem 1, with probability at least $1 - \tau/2$,
\begin{equation*}
V_{\hat{\pi}} \geq 
   \frac{1-\alpha}{1+\alpha}\hat{V}_{\hat{\pi}} 
	 - \frac{R_{\max}}{(1+\alpha)(1-\gamma)}
     \left[\frac{\ln(8 / \tau) + |\hat{\pi}|\ln(KD|A||Z|)}{\alpha K}\right].\nonumber\\
\end{equation*}
Suppose the above inequality holds on a random set of \nsam scenarios.
Note that there is a $\pi' \in \polclass_{b_0, \height,\nsam}$ which is a subtree of
$\pi$ and has the same trajectories on these scenarios up to depth $D$.
By the choice of $\hat{\pi}$, it follows that the following hold on the same
scenarios
\begin{equation}
V_{\hat{\pi}} \geq 
   \frac{1-\alpha}{1+\alpha}\hat{V}_{\pi'} 
	 - \frac{R_{\max}}{(1+\alpha)(1-\gamma)}
      \left[\frac{\ln(8 / \tau) + |\pi'|\ln(KD|A||Z|)}{\alpha K}\right].
			\nonumber\\
\end{equation}
In addition, $|\pi| \geq |\pi'|$, and $\hat{V}_{\pi'} \geq
\hat{V}_{\pi} - \gamma^D (R_{max})/(1-\gamma)$ since $\pi'$ and
$\pi$ only differ from depth $D$ onwards, under the chosen scenarios.
It follows that for these scenarios, we have
\begin{equation} \label{eqn:b1}
V_{\hat{\pi}} \geq 
	\frac{1-\alpha}{1+\alpha}\left(\hat{V}_{\pi}
		- \gamma^D\frac{R_{max}}{1-\gamma}\right) 
  - \frac{R_{\max}}{(1+\alpha)(1-\gamma)}\left[\frac{\ln(8 / \tau) + |\pi|\ln(KD|A||Z|)}{\alpha K}\right].
\end{equation}
Hence, the above inequality holds with probability at least $1 - \tau/2$.

By Lemma~\ref{lem:hoeffding}, with probability at least $1 - \tau/2$, we have
\begin{equation} \label{eqn:b2}
	\hat{V}_{\pi} \geq V_{\pi} - \frac{R_{\max}}{1-\gamma}\sqrt{\frac{2\ln (2/\tau)}{K}}.
\end{equation}

By the union bound, with probability at least $1 - \tau$, both
Inequality~(\ref{eqn:b1}) and Inequality~(\ref{eqn:b2}) hold, which imply
the inequality in the theorem holds.
This completes the proof.
\end{proof}

\subsection*{Proof of Theorem~\ref{th:bounded} and Theorem~\ref{th:unbounded}}
\begin{lemmastar} \hspace{-0.5ex}{\rm\bf \ref{lem:forward}}
For any \despot node $b$, if $\wdeu(b) > 0$ and $a^* =
  \arg\max_{a\in A} \uppernu(b, a)$, then
\[\wdeu(b) \le \sum_{z \in Z_{b, a^*}} \wdeu(b'),\]
where $b' = \tau(b,a^*,z)$ is a child of $b$. 
\end{lemmastar}

\begin{proof}
Let $\children(b, a^*)$ be all the set of all $b'$ of the form $b' = \tau(b,
a, z)$ for some $z \in Z_{b, a^*}$, that is, the set of children nodes of $b$
in the DESPOT tree.
If $\wdeu(b) > 0$, then $\uppernu(b) - \lowernu(b) > 0$, and thus $\uppernu(b) \neq
\lowernu_0(b)$.
Hence we have 
\[\uppernu(b) 
  = \uppernu(b, a^*)
  = \rho(b, a^*) + \sum_{b' \in \children(b, a^*)} \uppernu(b'), \text{ and }\]
\[\lowernu(b) \ge \lowernu(b, a^*) 
  \ge \rho(b, a^*) + \sum_{b' \in \children(b, a^*)} \lowernu(b'),\]
Subtracting the first equation by the second inequality, we have
\[ \uppernu(b) - \lowernu(b) \le \sum_{b' \in \children(b, a^*)} [\uppernu(b') - \lowernu(b')].\]
Note that 
\[ \frac{|\seqset_{b}|}{K} \cdot \xi \cdot \gap(b_0) = 
  \sum_{b' \in \children(b, a^*)} \frac{|\seqset_{b'}|}{K} \cdot \xi \cdot \gap(b_0).\]
We have
\[{ \sum_{b' \in \children(b, a^*)} [\uppernu(b') - \lowernu(b') - 
 \frac{|\seqset_{b'}|}{K} \cdot \xi \cdot \gap(b_0)] }
 \ge {\uppernu(b) - \lowernu(b) - \frac{|\seqset_{b}|}{K} \cdot \xi \gap(b_0)}.\]
That is, $\wdeu(b) \le \sum_{b' \in \children(b, a^*)} \wdeu(b')$.
\end{proof}

\begin{lemmastar} \hspace{-0.5ex}{\rm\bf \ref{lem:blocking}}
Let $b'$ be an ancestor of $b$ in a \despot $\dtree$ and $\len(b',b)$ be the number of nodes on the path from $b'$ to
$b$. If
\[
\frac{|\seqset_{b'}|}{K} \gamma^{\depth(b')} (\upperv(b') - \lowerv_0(b'))
\le \lambda \cdot  \len(b', b),
\]
then $b$ cannot be a belief node in an optimal regularized policy derived from
$\dtree$.
\end{lemmastar}
\begin{proof}
If $b$ is included in the policy $\pi$ maximizing the regularized utility,
then for any node $b'$ on the path from $b$ to the root, the subtree
$\pi_{b'}$ under $b'$ satisfies
\begin{align*}
 &\frac{|\seqset_{b'}|}{K} \gamma^{\depth(b')} \upperv(b') - \lambda |\pi_{b'}|  \\
 \ge& \frac{|\seqset_{b'}|}{K} \gamma^{\depth(b')} \hat{V}_{\pi_{b'}}(b') -
 \lambda |\pi_{b'}| &\text{(since $U(b') > \hat{V}_{\pi_{b'}}(b')$)} \\
 >& \frac{|\seqset_{b'}|}{K} \gamma^{\depth(b')} \lowerv_0(b'), &\text{(since $b'$ is not pruned)}
\end{align*}
which implies
\begin{equation*}
\frac{|\seqset_{b'}|}{K} \gamma^{\depth(b')} \upperv(b') - \lambda |\pi_{b'}|  
 > \frac{|\seqset_{b'}|}{K} \gamma^{\depth(b')} \lowerv_0(b').
\end{equation*}
Rearranging the terms in the above inequality, we have 
\begin{equation*}
\frac{|\seqset_{b'}|}{K} \gamma^{\depth(b')} [\upperv(b') - \lowerv_0(b')] 
  >   \lambda |\pi_{b'}|. 
\end{equation*}
We have $\len(b', b) \le |\pi_{b'}|$ as $b$ is not pruned and thus $\pi_{b'}$
need to include all nodes from $b'$ to $b$.
Hence,
\begin{equation*}
\frac{|\seqset_{b'}|}{K} \gamma^{\depth(b')} [\upperv(b') - \lowerv_0(b')]
  >  \lambda \cdot \len(b', b).
\end{equation*}
\end{proof}

\begin{theoremstar}\hspace{-0.5ex}{\rm\bf \ref{th:bounded}}
Suppose that \tmax is bounded and that the anytime \despot algorithm terminates
with a partial \despot $\dtree'$ that has gap $\gap(b_0)$ between the upper
and lower bounds at the root $b_0$. 
The optimal regularized policy $\hat{\pol}$ derived from $\dtree'$
satisfies
\begin{dm}
  \nu_{\hat{\pol}}(b_0) \ge \nu^*(b_0) - \gap(b_0) - \delta,
\end{dm}
where $\nu^*(b_0)$ is the value of an optimal regularized policy derived from the full
\despot \dtree at $b_0$. 
\end{theoremstar}
\begin{proof}
Let $U'_0(b) = \upperv_0(b) + \delta$, then $U'_0$ is an exact upper bound. 
Let $\uppernu'_0$ be the corresponding initial upper bound, and $\uppernu'$ be the
corresponding upper bound on $\nu^*(b)$.
Then $\uppernu'_0$ is a valid initial upper bound for $\nu^*(b)$ and the backup
equations ensure that $\uppernu'(b)$ is a valid upper bound for $\nu^*(b)$.
On the other hand, it can be easily shown by induction that 
$\uppernu(b) + \gamma^{\depth(b)} \frac{|\seqset_{b}|}{K} \delta \ge \uppernu'(b)$.
As a special case for $b = b_0$, we have $\uppernu(b) + \delta \ge \uppernu'(b_0)$.
Hence, when the algorithm terminates, we also have 
$\uppernu(b_0) + \delta \ge \uppernu'(b_0) \ge \nu^*(b_0)$.
Equivalently, 
$\nu_{\hat{\pi}} 
= \lowernu(b_0) \ge \nu^*(b_0) - (\uppernu(b_0) - \lowernu(b_0)) - \delta 
= \nu^*(b_0) - \gap(b_0) - \delta$.
The first equality holds because the initialization and the computation of the
lower bound $\lowernu$ via the backup equations are exactly that for finding
an optimal regularized policy value in the partial \despot.
\end{proof}

\begin{theoremstar}\hspace{-0.5ex}{\rm\bf \ref{th:unbounded}}
Suppose that \tmax is unbounded and $\gap_0$ is the target gap between the
upper and lower bound at the root of the partial \despot constructed by the
anytime \despot algorithm. Let $\nu^*(b_0)$ be
the value of an optimal policy derived from the full
\despot \dtree at $b_0$.
\begin{itemize}
\item[(1)] If $\gap_0 > 0$, then the algorithm terminates in finite
  time with a near-optimal policy $\hat \pol$  satisfying
\begin{dm}
\nu_{\hat{\pol}}(b_0) \ge \nu^*(b_0) - \gap_0 - \delta.
\end{dm}
\item[(2)] If $\gap_0 = 0$, $\delta=0$, and
  the regularization constant $\lambda>0$, then the 
  algorithm terminates in finite time with an optimal policy $\hat \pol$, \ie, 
   $\nu_{\hat{\pol}}(b_0) = \nu^*(b_0)$.
\end{itemize}
\end{theoremstar}

\begin{proof}
\medskip\noindent(1) 
It suffices to show that eventually $\uppernu(b_0) - \lowernu(b_0) \le
\epsilon_0$, which then implies $\nu_{\hat{\pol}}(b_0) \ge \nu^*(b_0) -
\epsilon_0 - \delta$ by Theorem~\ref{th:bounded}.

We first show that only nodes within certain depth $d_0$ will ever be expanded.
In particular, we can choose any $d_0 > 0$ such that $\xi \epsilon_0  >
\gamma^{d_0} R_{max} / (1 - \gamma)$.
For any node $b$ beyond depth $d_0$, we have
$\gap(b) 
\le \gamma^{\depth(b)} \frac{|\seqset_b|}{K} \frac{R_{max}}{1 - \gamma} 
\le \gamma^{d_0} \frac{|\seqset_b|}{K} \frac{R_{max}}{1 - \gamma} 
< \frac{|\seqset_b|}{K} \cdot \xi \cdot \epsilon_0$.
Since $\gap(b_0) > \epsilon_0$ before the algorithm terminates, we have 
$\wdeu(b) = \gap(b) - \frac{|\seqset_b|}{K} \cdot \xi \cdot \gap(b_0) 
< \gap(b) - \frac{|\seqset_b|}{K} \cdot \xi \cdot \epsilon_0 < 0$, and $b$
will not be expanded during the search.

On the other hand, the forward search heuristic guarantees that each
exploration
closes the gap of at least one node or expands at least one node.
To see this, note that an exploration terminates when its depth exceeds the limit
$D$, or when a node is pruned, or when it encounters a node $b$ with $\wdeu(b)
< 0$.
In the first two cases, the gap for the last node is closed.
We show that for the last case, the exploration must have expanded at least one
node.
By Lemma~\ref{lem:forward}, if $\wdeu(b) > 0$, then the next chosen node $b'$
satisfies $\wdeu(b') > 0$, otherwise the upper bound for $\wdeu(b)$ will be at
most 0, a contradiction.
Since the root $b_0$ has positive $\wdeu$, thus the exploration will follow a path
with nodes with positive $\wdeu$ to reach a leaf node, and then expand it.
Termination may then happen as the next chosen node has negative \wdeu.

Since each node can be expanded or closed at most once, but only nodes within
depth $d_0$ can be considered, thus after finitely many steps 
$\gap(b_0) = \uppernu(b_0) - \lowernu(b_0) \le \epsilon_0$, and the search
terminates.

\medskip\noindent(2)
Following the proof of (1), when $\epsilon = 0$, before $\uppernu(b_0) =
\lowernu(b_0)$, each exploration either closes the gap of at least one node, or expands
at least one leaf node.
However, both can only be done finitely many times because the Prune
procedure ensures only nodes within a depth of 
$\lceil \frac{R_{max}}{ \lambda (1-\gamma)} \rceil + 1$ can be expanded, and
the gap of each such node can only be closed once only.
Thus the search will terminate in finite time with $\uppernu(b_0) = \lowernu(b_0)$.
Since the upper bound is a true one, we have 
$\nu_{\hat{\pi}}(b_0) = \lowernu(b_0) = \nu^*(b_0)$, and thus $\hat{\pi}$ is an
optimal regularized policy derived from the infinite horizon \despot.
\end{proof}

\newpage
\section{Pseudocode for Anytime \despot}\label{sec:pseudocode}

\begin{algorithm}
\caption{Anytime \despot}\label{alg:despot-all}
\setlength{\columnsep}{25pt}
\setlength{\marginparwidth}{0pc}
\begin{multicols}{2}
\textbf{Input}
\begin{itemize}
  \itemsep=-3pt
\item[$\beta\colon$] Initial belief.
\item [$\gap_0\colon$] The target gap between $\uppernu(b_0)$ and $\lowernu(b_0)$.
\item [$\xi\colon$] The rate of target gap reduction. 
\item [$K\colon$] The number of sampled scenarios.
\item [$D\colon$] The maximum depth of the DESPOT.
\item [$\lambda\colon$] Regularization constant. 
\item [$\tmax\colon$] The maximum online planning time per step.
\end{itemize}

\begin{algorithmic}[1]
\STATE  Initialize $\bvar{b} \gets \beta$.
\LOOP
\STATE $\lowernu \gets \textsc{BuildDespot($\bvar{b}$)}$.
\STATE $a^* \gets \max_{a\in A} \lowernu(\bvar{b}, a)$.
\IF{$\lowerv_0(\bvar{b}) > \lowernu(\bvar{b}, a^*)$}
 	\STATE $a^* \gets \pi_0(\bvar{b})$
\ENDIF
\STATE Execute $a^*$.
\STATE Receive observation $z$.
\STATE $\bvar{b} \gets \tau(\bvar{b}, a^*, z)$.
\ENDLOOP
\end{algorithmic}
\bigskip
$\textsc{BuildDespot}({b}_0) $
\begin{algorithmic}[1]
\STATE Sample randomly a set $\seqset_{{b}_0}$ of $\nsam$  scenarios 
from the current belief ${b}_0$. 
\STATE Create a new \despot \dtree with a single node $\bvar{b}$ as the root.
\STATE Initialize $\upperv({b}_0), 
\lowerv_0({b}_0), \uppernu({b}_0), \text{ and } \lowernu({b}_0) $.
\STATE $\gap({b}_0) \gets \uppernu({b}_0) - \lowernu({b}_0)$. 
\WHILE {$\gap({b}_0) >  \gap_0$
          and the total running time is less than \tmax}
	\STATE $b \gets \textsc{Explore}(\dtree, \bvar{b}_{0})$.
\STATE $\textsc{Backup}(\dtree, b)$.
\STATE $\gap({b}_0) \gets \uppernu({b}_0) - \lowernu({b}_0)$. 
\ENDWHILE
\RETURN $\lowernu$

\end{algorithmic}
\columnbreak
$\textsc{Explore}(\dtree, b)$
\begin{algorithmic}[1]
\WHILE {$\depth(b) \leq D$,  $\wdeu(b) > 0$, and \textsc{Prune}($\dtree$, $b$) =
  {\small FALSE}}
\IF {$b$ is a leaf node in $\dtree$}
	\STATE Expand $b$ one level deeper.  Insert each new child $b'$ of $b$
        into $\dtree$, and
        initialize $\upperv(b')$,
        $\lowerv_0(b')$,  $\uppernu(b')$,   and  $\lowernu(b')$.
\ENDIF
\STATE $a^* \gets \argmax_{a \in A} {\uppernu(b, a)}$.
\STATE $z^* \gets \argmax_{z \in Z_{b,a^*}} \wdeu(\tau(b,a^*,z))$.
\STATE $b \gets \tau(b, a^*, z^*)$.
\ENDWHILE
\IF {$\depth(b) > D$}
	\STATE \textsc{MakeDefault}($b$).
\ENDIF
\RETURN $b$.
\end{algorithmic}
\medskip
\textsc{Prune}($\dtree, b$)
\begin{algorithmic}[1]
\STATE $\textsc{Blocked} \gets \text{\small FALSE}$.
\FOR{each node $x$ on the path from $b$ to the root of \dtree}
	\IF{$x$ is blocked by any ancestor node in \dtree}
		\STATE \textsc{MakeDefault}($x$);
		\STATE \textsc{Backup}(\dtree, $x$).
                \STATE $\textsc{Blocked} \gets \text{\small TRUE}$.
	\ELSE
		\STATE \textbf{break}
	\ENDIF
\ENDFOR
\RETURN \textsc{Blocked}
\end{algorithmic}
\medskip
\textsc{MakeDefault}($b$)
\begin{algorithmic}[1]
	\STATE $\upperv(b) \gets \lowerv_0(b)$.
	\STATE $\uppernu(b) \gets \lowernu_0(b)$.
	\STATE $\lowernu(b) \gets \lowernu_0(b)$.
\end{algorithmic}
\medskip
\textsc{Backup}($\dtree, b$)
\begin{algorithmic}[1]
\FOR{each node $x$ on the path from $b$ to the root of \dtree}
\STATE  Perform backup on $\uppernu(x) $, $\lowernu(x) $, and 
  $\upperv(x)$. 
\ENDFOR
\end{algorithmic}
\end{multicols}
\end{algorithm}

\bibliography{ref}

\end{document}